%% file: main.tex
\documentclass[conference]{IEEEtran}

\usepackage{times}
\let\labelindent\relax
\usepackage{enumitem}

\usepackage{adjustbox}
\usepackage[numbers]{natbib}
\usepackage{multicol}
\usepackage[bookmarks=true]{hyperref}
\usepackage{graphicx} 
\usepackage{hyperref}
\usepackage{xcolor}
\definecolor{LightGray}{gray}{0.95}
\usepackage{fancyvrb, minted}

\usepackage{amssymb}
\usepackage{amsmath}
\usepackage{graphicx}
\usepackage{caption}
\usepackage[dvipsnames]{xcolor}
\usepackage{subcaption}
\graphicspath{ {images/} }
\usepackage{float}
\usepackage{amsthm}
\usepackage{listings}
\usepackage{algorithm,algcompatible}
\makeatletter
\def\munderbar#1{\underline{\sbox\tw@{$#1$}\dp\tw@\z@\box\tw@}}
\makeatother

\usepackage{algorithm}
\usepackage{algorithmicx}
\usepackage{algpseudocode}
\usepackage{algpseudocode}
\usepackage{makecell}
\usepackage{booktabs}       
\usepackage{url}
\usepackage{multirow}
\usepackage{lipsum}
\usepackage{bbm}
\usepackage{soul}
\usepackage{mathtools}

\usepackage{tikz}
\usetikzlibrary{arrows.meta, positioning}

\usepackage{nomencl}
\makenomenclature

\newtheorem{theorem}{Theorem}

\newtheorem{remark}{Remark}

\usepackage{lipsum}  
\usepackage{stmaryrd}

\newcommand{\R}{\mathbb{R}}

\begin{document}

\title{Betting for Sim-to-Real Performance Evaluation}


\author{\authorblockN{Zaid Mahboob\authorrefmark{1}\authorrefmark{2},
Yujia Chen\authorrefmark{1}\authorrefmark{2},
and
Bowen Weng\authorrefmark{1}\authorrefmark{3}}
\authorblockA{\authorrefmark{1}Department of Computer Science\\
Iowa State University,
Ames, Iowa 50011\\ Email: \{zaidm, yjchen, bweng\}@iastate.edu}
\authorblockA{\authorrefmark{2}Equal contribution} \ \ \authorblockA{\authorrefmark{3}Corresponding author}}


%

\maketitle

\begin{abstract}
This paper studies the problem of robot performance evaluation, focusing on how to obtain accurate and efficient estimates of real-world behavior under severe constraints on physical experimentation. Such estimates are essential for benchmarking algorithms, comparing design alternatives, validating controllers, and supporting certification or regulatory decision-making, yet real-world testing with physical robots is often expensive, time-consuming, and safety-limited. To mitigate the scarcity of real-world trials, sim-to-real methodologies are commonly employed, using low-cost simulators to inform, supplement, or prioritize physical experiments. Departing from (and complementary to) existing approaches in variance reduction (e.g., importance-sampling variants) or bias-correction (e.g., through prediction-powered inference or learned control variates), we examine this performance-evaluation problem through the lens of \emph{betting}. We establish theoretical conditions under which a betting mechanism can yield accurate and efficient estimates (provably outperforming the Monte Carlo estimator) and we characterize how such bets should be constructed. We further develop theoretically grounded yet practically implementable approximations of the ideal bet, and we provide concrete decision rules that diagnose when these approximate betting strategies are working as intended. We demonstrate the effectiveness of the proposed methods using both synthetic examples and cross-fidelity computational simulators. Notably, we also showcase an illustrative case in which a group of synthetic distributions are used to infer the real-world pick-and-place accuracy of a robotic manipulator, a seemingly unconventional sim-to-real transfer that becomes natural and feasible under the proposed betting perspective. Programs for reproducing empirical results are available at \href{https://github.com/ISUSAIL/Bet4Sim2Real}{https://github.com/ISUSAIL/Bet4Sim2Real}.
\end{abstract}

\IEEEpeerreviewmaketitle

\section*{DISCLAIMER\footnote{This work was supported in part by Presidential Interdisciplinary Research Seed Grant at Iowa State University and in part by financial assistance award 70NANB25H125 from U.S. Department of Commerce, National Institute of Standards and Technology.}\footnote{A modified version of this work has been accepted for publication at Robotics: Science and Systems (RSS) 2026. This preprint mostly preserves the manuscript in its original form, as initially submitted to RSS for review. Zaid Mahboob and Yujia Chen are co-first authors with equal contribution; their listed order is reversed in the published RSS version.}}
In this paper, terms such as ``gambler'', ``gambling'', ``betting'', and ``bets'' are used strictly as theoretical metaphors within the context of robot testing and performance evaluation. They refer only to abstract mathematical constructs for sequential decision-making and uncertainty modeling. Under no circumstances does this work involve, promote, or relate to real-world financial gambling, monetary transactions, or any form of betting activity, within or outside the robotics context.

\section{Introduction}
Although rarely stated explicitly, roboticists often behave much like ``gamblers'', at least occasionally, and in a constructive sense. One develops theory-inspired or data-driven algorithms, controllers, software pipelines, and hardware stacks largely within simulators or synthetic computational-aided environments, and then deploys them on real machines with the hope, though rarely with a formal guarantee, that performance will carry over. This transfer is guided by an implicit and continually updated notion of confidence: as real-world data accumulate, our trust in the simulator's predictions is revised, reinforced, or corrected. To the best of the authors' knowledge, this paper represents the first attempt to \emph{formalize and theorize} this longstanding practice of ``\emph{betting}'' within sim-to-real workflows in robotics.

\subsection{The sim-to-real performance evaluation problem}
We focus on a specific yet fundamental aspect of sim-to-real transfer: \emph{performance evaluation}~\cite{feng2021intelligent,kadian2020sim2real,muratore2019assessing,weng2024towards,Weng2025Rethink}. Given a \emph{fixed} policy, controller, or robotic system, whether model-based, learning-based, or hybrid, as it stands, one seeks to \emph{efficiently} obtain an \emph{accurate} estimate of its expected performance under a certain real-world environment. This challenge arises in almost every corner of robotics: benchmarking learning algorithms~\cite{agarwal2021deep,duan2016benchmarking,kress2024robot}, evaluating safety-critical behaviors~\cite{weng2023comparability,kalra2016driving}, validating controllers~\cite{feng2023dense,roth1976performance}, comparing algorithmic or policy variants~\cite{vincent2024generalizable,colas2019hitchhiker}, or supporting certification and regulatory processes~\cite{Weng2025Rethink,roth1976performance,ansi_ria_r1505,iso9283,van2017euro}. 

Formally speaking, consider a probability distribution $P$ over the measurable space $(\mathcal{X},\mathcal{F})$. Let us further assume, albeit somewhat unrealistically in general but quite common in robotics, that $P$ represents an \emph{expensive} real-world data-generating distribution of various states concerning the robot being tested as well as environmental factors.

Given a bounded performance scoring function $\psi: \mathcal{X} \rightarrow \mathcal{M} \subset \R$, consider the performance evaluation target of inference as the scalar mean
\begin{equation}\label{eq:mu}
    \mu \triangleq \mathbb{E}_{x \sim P}[\psi(x)].
\end{equation}
When $\psi$ is a Boolean indicator of failure versus non-failure, $\mu$ reduces to the risk (i.e., the failure probability)~\cite{feng2021intelligent,Weng2025Rethink,kalra2016driving}. In practice, many other performance measures fall within the same expectation-based formulation, such as normalized rewards~\cite{sutton1998reinforcement,peng2018sim}, bounded stability scores~\cite{weng2023comparability, hyon2006passivity}, or clipped cost signals. W.l.o.g., we take $\mathcal{M} = [0,1]$ (via an affine rescaling/normalization of $\psi$ if necessary~\cite{boyd2004convex}) for the remainder of the analysis.

In general, to estimate the mean $\mu$, we have the Monte Carlo inference~\cite{metropolis1949monte,metropolis1953equation,hastings1970monte} as
\begin{equation}\label{eq:mc}
    \hat{\mu}_{\text{MC}} = \frac{1}{n}\sum_1^n \psi(x_i), x_i \sim P \text{ i.i.d.}
\end{equation}
Despite the theoretical completeness of Monte Carlo estimation (e.g., $\lim_{n\rightarrow\infty}\hat{\mu}_{\mathrm{MC}}=\mu$ almost surely~\cite{metropolis1953equation,hammersley1964monte}), plain Monte Carlo is often not a practically ``efficient'' evaluation method, especially when extreme-performance outcomes occur with very small probability under $P$ (the well-known long-tail or ``curse of rarity'' phenomenon~\cite{liu2024curse}). This inefficiency becomes particularly pronounced when $P$ denotes the distribution of states induced by \emph{real-world} trials on \emph{real robots}, where tests are costly, time-consuming, safety-constrained, and often subject to hardware wear, environmental variability, or operational risk~\cite{weng2024towards, kalra2016driving, afzal2020study,koopman2016challenges}. 

\emph{Simulators}, therefore, could play a crucial role. They offer abundant, inexpensive, and controlled observations that can supplement or guide real-world testing. Consider a parameterized \emph{simulator} as a family of other distributions $\{\!Q_{\theta}\!:\!\theta\!\in\!\Theta\!\}$ over the same measurable space mentioned above. By configuring $\theta$ differently, one can have different distributions of $Q$. Contrary to the real-world distribution $P$, let's further assume $Q_{\theta}$ corresponds to a family of relatively ``cheaper'' distributions from which we can affordably draw large numbers of samples. This setup motivates our central question: \emph{can we leverage the low cost and flexible configurability of $Q_{\theta}$ to perform accurate and efficient inference on $P$ with theoretical guarantees?}


\subsection{Literature review}
We note that this paper is not the first attempt to address the above research question. A substantial body of prior work has explored related directions, generally falling into two complementary perspectives~\cite{asmussen2007rare}: (i) \emph{variance reduction}, and (ii) \emph{bias correction}. 

The variance-reduction perspective typically relies on \emph{weighted sampling}, with \emph{importance sampling} (IS) as a canonical example~\cite{kahn1951estimation,asmussen2007stochastic,o2018scalable}:
\begin{equation}\label{eq:is}
    \hat{\mu}_{\text{IS}} = \frac{1}{n}\sum_{i=1}^n \psi(x_i)\frac{p(x_i)}{p'(x_i)}, 
    \qquad x_i \sim P' \text{ i.i.d.}
\end{equation}
Here, the auxiliary distribution \(P'\) is chosen to oversample regions that contribute most to \(\mu\). Under standard regularity conditions~\cite{asmussen2007stochastic,chatterjee2018sample}, IS yields unbiased estimators with reduced variance relative to vanilla Monte Carlo~\eqref{eq:mc}. In robotics, however, both \(P\) and \(P'\) typically correspond to real-world data collection, with the simulator \(Q_\theta\) serving mainly to guide the construction of a better \(P'\)~\cite{feng2021intelligent,weng2023comparability, feng2023dense,koren2018adaptive}. Consequently, the most effective simulator under this view is one that closely matches \(P\), since tighter alignment enables more accurate variance reduction.

Bias-correction approaches instead seek to implicitly~\cite{chebotar2019closing} or explicitly~\cite{kadian2020sim2real} model and correct the discrepancy between the real distribution \(P\) and a simulator \(Q_\theta\). These methods interpret the bias $\mathbb{E}_{P}[\psi] - \mathbb{E}_{Q_\theta}[\psi]$ as a structured, sample-dependent quantity that can be learned or parameterized. Concretely, one estimates a correction function $b_\theta(x)$, for example via paired real–sim or predictor–real evaluations as in prediction-powered inference (PPI)~\citep{angelopoulos2023prediction, badithela2025reliable}, or via learned correlators or control variates~\cite{ranganath2014black,luo2025leveraging}. The correction is then used to construct a bias-adjusted estimator, either by evaluating corrected simulator samples or by calibrating predictor-based estimates using real data, yielding an estimator of the generic form
\begin{equation}
    \hat{\mu}_{\mathrm{BC}}
    =
    \frac{1}{n}\sum_{i=1}^{n}\bigl(\psi(x_i) + b_\theta(x_i)\bigr),
    \qquad x_i \sim Q_\theta \text{ i.i.d.},
\end{equation}
where this expression represents a canonical simulator-centric bias-correction form; predictor-based variants such as PPI can be written in an equivalent correction-based form. Although conceptually distinct from IS, bias-correction methods share a similar implication: simulator effectiveness hinges on fidelity to $P$. The closer \(Q_\theta\) aligns with the real distribution, the more reliable the correction \(b_\theta\) becomes.

Note neither approach mentioned above fundamentally exploits the scalability of modern robotic simulation: simulated data grows cheaply and rapidly, while real-world data remains scarce, preventing simulation abundance from translating into proportional gains in reliable real-world performance estimation~\cite{kalra2016driving,koopman2016challenges,asmussen2007stochastic,angelopoulos2023prediction}. Moreover, from a theoretical standpoint, these lines of work suggest a common conclusion: for \emph{performance evaluation}, a ``good'' simulator is one that closely approximates reality. Yet this conclusion contrasts sharply with another highly successful paradigm in sim-to-real robotics. In \emph{policy transfer}, where the policy is continually learned and adapted in the hope of achieving high performance in the real world, ``wrong'' but diverged simulators have enabled dramatic progress by deliberately embracing mismatch. Techniques like domain randomization~\cite{peng2018sim,tobin2017domain,tan2018sim,muratore2021data,akkaya2019solving,huang2024collision,pan2009survey}, adversarial perturbations~\cite{weng2023comparability, pinto2017robust}, and robustness-oriented training inject structured variability to promote generalization across the reality gap, despite limited theoretical understanding~\cite{kadian2020sim2real, muratore2021data}.



What if the very ingredients that drive the success of robust sim-to-real policy transfer, its deliberate distributional mismatch, structured variability, and controlled uncertainty, are in fact \emph{beneficial} for performance evaluation as well, and in certain regimes even \emph{preferable} to perfect simulator alignment? Such a possibility appears incompatible with classical variance-reduction or bias-correction viewpoints, which rely on minimizing or compensating for simulator–reality discrepancies. The \emph{betting} perspective introduced in this paper suggests a different interpretation. Within a betting framework, predictive advantage, rather than perfect fidelity, is what matters. A simulator that is ``wrong but informative'', because it induces directional predictive signals or structured uncertainty, may provide a stronger edge than one that attempts to mimic reality exactly but carries no exploitable variance. This re-framing opens the door to a unified viewpoint in which tools traditionally associated with policy transfer (e.g., domain randomization) may also play a principled and theoretically grounded role towards (provably) accurate and efficient sim-to-real \emph{performance evaluation}.

\subsection{Main contributions}
To the best of the authors' knowledge, this paper makes the first attempt in formulating the sim-to-real policy evaluation problem as a \emph{betting} problem. 
\begin{itemize}[leftmargin=*, labelsep=0.5em]
    \item We first introduce a \emph{sequential betting formalism} (Algorithm~\ref{alg:abstract_betting}) for sim-to-real performance inference, in which simulator-informed bets induce a \emph{bet-weighted estimator} of the target performance measure~\eqref{eq:mu}. 
    \item We then show that, under this mechanism, the classical Kelly framework for optimal betting approximately \emph{coincides} with the statistically optimal strategy for efficient estimation (Theorems~\ref{thm:efficiency} and \ref{thm:kelly-equivalence}). That is, the strategy that maximizes wealth growth through betting is exactly the one that achieves the best sampling efficiency and estimation accuracy for the robot’s real-world performance under our proposed framework.
    \item We translate these theoretical insights into a practically executable sim-to-real testing algorithm inspired by \emph{Cover's universal portfolio principle} (Algorithm~\ref{alg:exp_betting}). By maintaining and adaptively reweighting a bank of randomized simulators using proper scoring rules, the proposed method systematically approximates the classical Kelly betting strategy.
    \item We further prove that wealth growth itself provides a diagnostic signal for assessing the quality of this approximation (Theorem~\ref{thm:wealth-evalue}), thereby endowing the practical algorithm with theoretical performance guarantees in real settings.
    \item Finally, we demonstrate the effectiveness and robustness of the proposed framework through a series of examples, ranging from controlled synthetic experiments to real robotic applications in legged locomotion and manipulation across both simulated and physical environments.
\end{itemize}
The proposed framework fundamentally leverages the scalability of modern simulation in robotics and offers a new theoretical perspective on understanding the role of domain randomization and simulation variance in real-world performance evaluation.

\noindent\textbf{Notations}. The set of real numbers are denoted by $\R$. $\mathbb{Z}$ denotes the set of positive integers and $\mathbb{Z}_k = \{1,\ldots,k\}$ for some $k\in \mathbb{Z}$. $|S|$ denotes the cardinality of a set $S$. $a \wedge b = \min(a,b)$ for some $a,b \in \R$.

\section{Main Method}\label{sec:method}
\begin{figure}[h]
    \centering
    \begin{adjustbox}{width=0.49\textwidth}
    \begin{tikzpicture}[>=stealth, every node/.style={align=center}]

\node (L1) at (4.5,3.5) {Section~\ref{subsec:method-algorithm} \\ Algorithm~\ref{alg:exp_betting}};
\node (R1) at (7,2) {Section~\ref{subsec:method-theorem} \\ Theorem~\ref{thm:efficiency}};

\node (L2) at (0,5.5) {Section~\ref{subsec:method-kelly} \\ Theorem~\ref{thm:kelly-equivalence}};
\node (R2) at (7,5.5) {Section~\ref{subsec:method-abstract} \\ Algorithm~\ref{alg:abstract_betting}};

\node (L3) at (0,2) {Section~\ref{subsec:method-assess} \\ Theorem~\ref{thm:wealth-evalue}};


\draw[->, line width=1.2pt] (L2) -- (L1)
  node[midway, sloped]
  {\footnotesize How to bet,\\ \footnotesize approximately?};

\draw[->, line width=1.2pt] (R1) -- (R2)
  node[midway, sloped]
  {\footnotesize When does betting \\ \footnotesize help, theoretically?};


\draw[->, line width=1.2pt] (L2) -- (R2)
  node[midway, above, sloped]
  {\footnotesize How to bet, theoretically?};

\draw[dashed,<->, line width=1.2pt] (R2) -- (L1)
  node[midway, above=-0.05cm, sloped]
  {};

\draw[->, line width=1.2pt] (L3) -- (L1)
  node[midway, sloped]
  {\footnotesize When does betting \\ \footnotesize help, approximately?};





\end{tikzpicture}
\end{adjustbox}
    \caption{\footnotesize{A roadmap of Section~\ref{sec:method}. Theoretical results (Theorems \ref{thm:efficiency}-\ref{thm:wealth-evalue}) establish when and how betting improves estimation; algorithms translate these insights into practice. Arrows show logical dependencies.}}
    \label{fig:overview}
    \vspace{-3mm}
\end{figure}
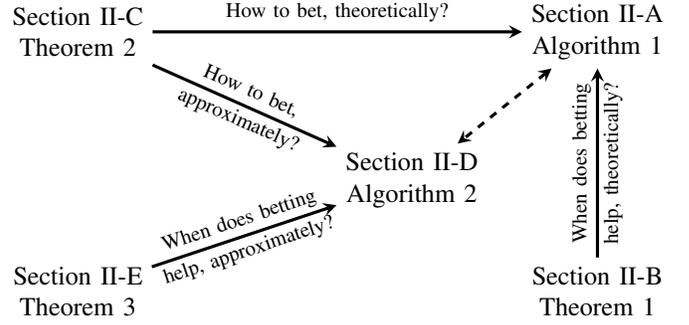

Proofs of all Theorems in this section can be found in the supplementary material.

\subsection{An abstract sequential betting algorithm}\label{subsec:method-abstract}

\begin{algorithm}[hb]
    \begin{algorithmic}[1]
\State {\bf Given:} $P$, $Q_\theta$, $\psi: \mathcal{X} \to \mathcal{M}$, $T \in \mathbb{Z}$
\State {\bf Initialize:} $\mathcal{D}\!=\!\emptyset$, $\tau_0 \in \mathcal{M}$, Wealth $W=1$, $S_0 = \emptyset$
\State {\bf For} $t = 1, 2, \ldots, T$:
\State \ \ \ \ Sample $\tilde{S}_t \sim Q_\theta$, $S_t \leftarrow S_{t-1} \cup \tilde{S}_t$ 
\State \ \ \ \ Based on $S_t$ and $\mathcal{D}$:
\State \ \ \ \ \ \ compute bet $b_t^{\uparrow} \in [0,1]$ on $\psi(x_{t})>\tau_{t-1}$
\State \ \ \ \ \ \ compute bet $b_t^{\downarrow} \in [0,1]$ on $\psi(x_{t})<\tau_{t-1}$
\State \ \ \ \ $b_t = \max\{b_t^{\uparrow}, b_t^{\downarrow}\}$
\State \ \ \ \ {\bf If} $b_t > 0$:
\State \ \ \ \ \ \ \ \ Sample $x_t \sim P$ 
\State \ \ \ \ \ \ \ \ $\mathcal{D}.\texttt{append}((x_t, b_t))$
\State \ \ \ \ \ \ \ \ $Y_t =
    \begin{cases}
        \ \ |\psi(x_{t})-\tau_{t-1}|, & \text{if the bet is correct},\\
        -|\psi(x_{t})-\tau_{t-1}|, & \text{otherwise}.
    \end{cases}$
\State \ \ \ \ \ \ \ \ $W \leftarrow W(1 +b_tY_t)$
\State \ \ \ \ \ \ \ \ $\tau_t = \sum_{(x,b) \in \mathcal{D}} b \cdot \psi(x) / \sum_{(x,b) \in \mathcal{D}} b $
\State {\bf Output:} $\hat{\mu}_{BW} = \tau_t$
\end{algorithmic}
    \caption{Abstract Betting}
    \label{alg:abstract_betting}
\end{algorithm}

An abstract version of the proposed sequential betting algorithm for sim-to-real performance inference is given in Algorithm~\ref{alg:abstract_betting}. It places a standard \emph{betting} mechanism at its core. At each round $t$, before observing the real sample but after drawing auxiliary information from the accumulated simulator samples, a bet $b_t$ is chosen as a fraction of the (future) payoff, reflecting the algorithm's belief about the correctness of a forthcoming prediction (lines 5–8). If the prediction is subsequently verified by the real outcome, the gambler’s \emph{wealth} increases by the corresponding payoff $Y_t$; if not, the wealth decreases proportionally (lines 12–13). If $b_t = 0$ for all $t$, the bet-weighted estimator is undefined; setting $b_t = 1$ for all $t$ recovers the Monte Carlo estimator. It is important to point out that this \emph{betting} setup does reflect common practices in robotic sim-to-real workflows, in some conceptual respects. Practitioners routinely decide whether real-world tests are worthwhile based on simulation confidence~\cite{kadian2020sim2real, muratore2021data,rai2019using}. As real-world samples accumulate, it also jointly affects the utilization of simulator and the resultant bet. 

On the other hand, the algorithm also departs from standard betting practice and robot performance evaluation in two important ways. The first atypical component is the \emph{double betting} mechanism (lines~5-8) (conceptually close to~\cite{howard2020time}), which determines a ``directional'' bet on the correlation between the upcoming real-world sample and the running bet-weighted mean of the previously observed real-world measurements. This mechanism is intentionally introduced to align the betting procedure with the performance–estimation objectives of this paper. Moreover, for the same reason, the algorithm outputs a \emph{bet-weighted estimate} (line~15) in place of the more traditional equally weighted estimator such as the Monte Carlo mean in~\eqref{eq:mc}. 

Finally, Algorithm~\ref{alg:abstract_betting} is not immediately executable as stated, since critical implementation details remain unspecified. The determination of bet sizes $b_t$ is described only conceptually as being ``based on $S_t$ and $\mathcal{D}$'' without concrete formulas. Despite these gaps (which shall all be detailed in Section~\ref{subsec:method-algorithm} later), we first seek to address a more fundamental question: under what conditions does the sequential betting framework, as abstractly described in Algorithm~\ref{alg:abstract_betting}, outperform the Monte Carlo estimator? Our approach is to impose assumptions directly on the \emph{properties} of the betting sequence rather than prescribing specific mechanisms for generating these bets. \textbf{This separates the question of \emph{(i) when betting helps} (a statistical question about bet properties) from \emph{(ii) how to generate good bets} (an algorithmic design question)}. 


\subsection{When does betting help, theoretically?}\label{subsec:method-theorem}
Regardless of how the sequence of bets is generated, the following theorem establishes a universal performance guarantee concerning bet-weighted performance estimate. It precisely characterizes when the bet-weighted estimator (line~15 of Algorithm~\ref{alg:abstract_betting}) achieves higher efficiency, and therefore greater accuracy, than the standard Monte Carlo estimator~\eqref{eq:mc}.

\begin{theorem}[Bet-weighted inference V.S. Monte Carlo efficiency]
\label{thm:efficiency}
Consider the abstract sequential betting algorithm (Algorithm~\ref{alg:abstract_betting}) with the bet-weighted estimator $\hat{\mu}_{\text{BW}} = \sum_{t=1}^T w_t \psi(x_t)$ where $w_t = b_t/\sum_j b_j$. Correspondingly, the Monte Carlo estimator $\hat{\mu}_{\text{MC}}$ is defined as in~\eqref{eq:mc} and the target mean $\mu$ in \eqref{eq:mu}. For a fixed sample budget $n$ from $P$, the bet-weighted estimator achieves lower expected squared error $\mathbb{E}[(\hat{\mu}_{\text{BW}} - \mu)^2] < \mathbb{E}[(\hat{\mu}_{\text{MC}} - \mu)^2]$ if and only if:
\begin{equation}\label{eq:var_red}
\underbrace{\frac{\sigma^2}{n} - \sum_{t=1}^n w_t^2 \sigma_t^2}_{\text{Variance Reduction}} > \underbrace{\left[\sum_{t=1}^n w_t \beta(w_t)\right]^2}_{\text{Bias Penalty}},
\end{equation}
where $\sigma_t^2 = \text{Var}(\psi(x_t) | w_t)$ is the conditional variance given the weight, $\beta(w) = \mathbb{E}[\psi(x_t) | w_t = w] - \mu$ is the weight–outcome dependence term, and $\sigma^2 = \text{Var}_{x \sim P}[\psi(x)]$ is the variance of the measure function from the distribution $P$.
\end{theorem}

To this end, the central question becomes how to design sequential double–betting procedures that realize the variance–reduction guarantees established in Theorem~\ref{thm:efficiency}. Concretely, this requires instantiating the bet-generation components left unspecified in Algorithm~\ref{alg:abstract_betting}.

\subsection{How to make good bets, theoretically?}\label{subsec:method-kelly}
If the discussion were purely about betting rather than robotics, the answer would be straightforward: use the classical Kelly principle~\cite{kelly1956new,thorp1975portfolio}. Kelly betting, originally developed in information theory~\cite{kelly1956new} and later adopted in finance and gambling~\cite{thorp1975portfolio,thorp2011understanding}, addresses a simple yet fundamental question: \emph{``When repeatedly placing bets that are slightly favorable, what fraction of my wealth should I wager each time so that my wealth grows as quickly as possible while controlling risk?''}

The central insight is that one should wager \emph{more} when the perceived advantage is high and \emph{less} when the outcome uncertainty is large. In its most compact form, the Kelly rule prescribes $\text{bet size} \;\propto\; \text{advantage}/\text{uncertainty (variance)}$, which is precisely an inverse-variance scaling~\cite{thorp2011understanding,rotando1992kelly}. 

A slightly more formal and sequential view of the Kelly principle comes from a small-stakes (Taylor-expanded) and small-edge approximation applied conditionally on past information~\cite{kelly1956new,cover1999elements,maclean1992growth}. Consider a betting process similar to Algorithm~\ref{alg:abstract_betting} in which, at round $t$, a fraction $b_t\in[0,1]$ is wagered and a random payoff $Y_t\in\R$ is received. Conditioned on the past history $\mathcal{F}_{t-1}$, the log-wealth increment admits the second-order Taylor expansion
\begin{equation}
    \log(1 + b_t Y_t)
    \;=\;
    b_t Y_t
    - \frac{1}{2} b_t^2 Y_t^2
    + O(|b_t Y_t|^3),
\end{equation}
which is accurate when the small-stakes condition $|b_tY_t|\ll1$ holds. Taking conditional expectations and retaining the leading terms yields the approximate conditional log-growth
\begin{equation}\label{eq:kelly-growth-conditional}
    b_t\,\mathbb{E}[Y_t \mid \mathcal{F}_{t-1}]
    \;-\;
    \frac{1}{2} b_t^2\,\mathbb{E}[Y_t^2 \mid \mathcal{F}_{t-1}].
\end{equation}
Differentiating and setting the derivative to zero gives
\begin{equation}\label{eq:kelly-small-stakes-inv-var}
    b_t^*
    \;\approx\;
    \frac{\mathbb{E}[Y_t \mid \mathcal{F}_{t-1}]}
         {\mathbb{E}[Y_t^2 \mid \mathcal{F}_{t-1}]}=\frac{\mu_t}{\sigma_t^2 + \mu_t^2}\approx\frac{\mu_t}{\sigma_t^2},
\end{equation}
where the final approximation follows under the small-edge condition $\mu_t^2 \ll \sigma_t^2$~\cite{maclean1992growth}. Note the small-stake and small-edge assumptions are classical in the information-theoretic and financial settings where the Kelly criterion was originally developed~\cite{kelly1956new,cover1999elements}. In the robotics performance-evaluation setting considered here, these conditions are even easier to satisfy, since the ``wagers'' are not real monetary stakes but mathematical constructs that can be scaled arbitrarily without affecting the underlying performance measure.

The remainder of this section will show that this same principle underlies the optimal variance-reduction strategy in our setting. In fact, we will see that the Kelly-optimal betting rule approximately coincides with the choice of weights that minimizes the estimation error of the bet-weighted mean, thereby connecting successful wealth growth to improved statistical efficiency of Algorithm~\ref{alg:abstract_betting}. 

\begin{theorem}[Kelly-style betting induces optimal bet-weighted inference]
\label{thm:kelly-equivalence}
Let $\hat{\mu}_{\mathrm{BW}} = \sum_{t=1}^T w_t \,\psi(x_t)$ be the bet-weighted estimator of Theorem~\ref{thm:efficiency}, where $w_t = b_t / \sum_{j=1}^T b_j$ and $\sigma_t^2 = \text{Var}(\psi(x_t)\mid \mathcal{F}_{t-1})$.
Suppose that at each round $t$ the bet size $b_t$ is chosen according to the Taylor-expanded Kelly rule~\eqref{eq:kelly-small-stakes-inv-var}, so that $b_t$ is proportional to $\mathbb{E}[Y_t\mid\mathcal{F}_{t-1}]/\text{Var}(Y_t\mid\mathcal{F}_{t-1})$, which we denote by $\kappa_{t-1}/\sigma_t^2$ with $\kappa_{t-1}:=\mathbb{E}[Y_t\mid\mathcal{F}_{t-1}]$.
Then the induced normalized weights satisfy $w_t \propto \kappa_{t-1}/\sigma_t^2$.
Moreover, if the predictive edge is (approximately) constant, i.e., $\kappa_{t-1}\approx \kappa$ for all $t$, then the induced weights reduce (approximately) to inverse-variance weighting, $w_t \propto 1/\sigma_t^2$, hence they (approximately) maximize the variance-reduction term $\sigma^2/T - \sum_{t=1}^T w_t^2 \sigma_t^2$ in~\eqref{eq:var_red}.

\end{theorem}

Theorem~\ref{thm:kelly-equivalence} applies to a broad class of Kelly-style updates with $\kappa_{t-1}$ as an $\mathcal{F}_{t-1}$-measurable ``edge'' factor and $\sigma_t^2$ as the conditional outcome variance. The constant-edge approximation $\kappa_{t-1}\approx\kappa$ holds most accurately in early-to-mid rounds when $|\tau_t - \mu|$ remains large; as $\tau_t \rightarrow \mu$, the edge diminishes and Kelly betting naturally reduces stakes, transitioning the estimator toward stability rather than aggressive weighting. In the specific case arising in Algorithm~\ref{alg:abstract_betting}, the payoff is defined as $Y_t \;=\; B_t\,(\psi(x_t) - \tau_{t-1})$, where $B_t\in\{+1,-1\}$ encodes whether the algorithm bets ``up'' or ``down''. Under this \emph{regression}-style payoff, $\text{Var}(Y_t \mid \mathcal{F}_{t-1})
    = \text{Var}(\psi(x_t)\mid\mathcal{F}_{t-1})
    = \sigma_t^2$,
so the Kelly denominator coincides exactly with the variance term used in Theorem~\ref{thm:efficiency}. The Kelly ``edge'' then becomes
\begin{equation}
        \kappa_{t-1}
        = B_t\bigl(\mathbb{E}[\psi(x_t)\mid\mathcal{F}_{t-1}] - \tau_{t-1}\bigr)
        = B_t (\mu-\tau_{t-1}).   
\end{equation}
A practical formulation of Kelly betting is expressed as
\begin{equation}\label{eq:ideal_kelly}
    b_t = \lambda \frac{\kappa_{t-1}}{\sigma_t^2}.
\end{equation}
$\lambda \in (0,1]$ is known as the \emph{Kelly fraction}, controlling the level of risk exposure~\cite{ziemba2006good}. The choice $\lambda = 1$ corresponds to the \emph{full Kelly} strategy, which maximizes asymptotic log-wealth growth~\cite{thorp2011understanding}, while $\lambda = 0.5$ is often referred to as the \emph{half-Kelly} strategy, offering more conservative behavior with reduced variance and improved robustness in finite-sample settings.


\begin{remark}\label{rmk:kelly-best-not-real}
It may be tempting to interpret the ideal Kelly update~\eqref{eq:ideal_kelly} as implicitly selecting the ``true'' real-world simulator, since the optimal bet~\eqref{eq:ideal_kelly} uses the real distribution. This interpretation is \emph{incorrect}. Kelly betting requires only the \emph{payoff distribution} of $\psi(x_t)-\tau_{t-1}$, specifically its conditional mean and variance, and does not depend on any generative or dynamical model of the real system. Thus, even in the ideal case, the object that Kelly optimizes against is not a simulator of reality, but the one-dimensional payoff law. 
\end{remark}

From a practical standpoint, however, the ``ideal'' Kelly rule~\eqref{eq:ideal_kelly} is not directly implementable, since neither the true predictive advantage nor the conditional variances $\sigma_t^2$ are known a priori. More ironically, this creates a circular dependency: constructing an accurate estimate of $\mu$ would itself require prior knowledge of $\mu$. This limitation was already implicit in Kelly's original formulation~\cite{kelly1956new}, which assumes that the gambler knows the true odds or outcome probabilities, an assumption rarely satisfied in practice. In the sim-to-real performance evaluation problem studied in this paper, this is precisely where the simulator re-enters the picture. 

\subsection{How to make good bets, approximately?}
\label{subsec:method-algorithm}
In the attempt to \emph{approximate} the ``ideal'' Kelly conditions, our main proposal is, again, inspired by classical ideas in finance and gambling, most notably \emph{Cover’s universal portfolio principle}~\cite{cover1991universal,cover2002universal}. By maintaining and updating a bank of experts, each representing a distinct simulator or predictive hypothesis, the algorithm adaptively reweights these experts according to their observed betting performance. This universal-portfolio–style aggregation admits a classical guarantee~\cite{algoet1988asymptotic,littlestone1994weighted,cesa2006prediction}: its cumulative performance is within logarithmic regret of the best expert in the bank~\cite{littlestone1994weighted}, despite that expert being unknown a priori (formal adaptation to our setting is left to future work).

Specifically, consider a bank of simulator instances $\{Q_{\theta^{(k)}}\}_{k=1}^K$ with $K$  ``expert'' betting strategies. The $k$-th expert can be instantiated by any lightweight simulator configuration or update rule, for example, a domain-randomized parameter setting, an adversarially stressed model instance, or any other plausible (or even not necessarily plausible) hypothesis about the real-world dynamics. Each expert proposes a predictive distribution for the performance measure, summarized by $\mu^{(k)}\!:=\!\mathbb{E}_{Q_{\theta^{(k)}}}[\!\psi(x)\!]$ and $(\sigma^{(k)})^2\!:=\!\text{Var}_{Q_{\theta^{(k)}}}\!(\!\psi(x)\!)$. The universal portfolio idea then maintains a \emph{score-weighted mixture} over all experts, allowing the algorithm to asymptotically compete with the best simulator in hindsight without knowing in advance which simulator is informative. Consistent with both Cover’s construction and classical gambling theory, each simulator is assigned a nonnegative weight $\pi_t^{(k)}$ with $\sum_k \pi_t^{(k)}=1$. A weighted simulator mixture is then used as an approximation of the ideal Kelly as
\begin{equation}\label{eq:approx-kelly}
    m_t\!=\!\sum_{k=1}^K\!\pi_t^{(k)}\! \mu^{(k)}\!,\!v_t\!=\!\sum_{k=1}^K\!\pi_t^{(k)}\!(\!\sigma^{(k)}\!)^2,b_t\!=\!
    \frac{\lambda\!|m_t\!-\!\tau_{t-1}|}{v_t}
    \!\wedge\!1,
\end{equation}
with betting direction $B_t=\mathrm{sign}(m_t-\tau_{t-1})$. Here $\lambda\in(0,1]$ is the Kelly fraction. This produces a \emph{single combined bet} at each round, exactly as in Algorithm~\ref{alg:abstract_betting}.

After observing the real-world outcome $y_t = \psi(x_t)$, for each simulator instance $k$, we maintain a cumulative log-score $s_t^{(k)}$ updated as
\begin{equation}\label{eq:score_update}
    s_t^{(k)}
    =
    s_{t-1}^{(k)}
    -
    \eta\,\ell\!\left(y_t;\mu^{(k)},(\sigma^{(k)})^2\right),
\end{equation}
where $\eta>0$ is a positive learning-rate (inverse-temperature) parameter that controls how aggressively evidence from real-world outcomes is accumulated, analogous to the learning-rate or risk-sensitivity parameters in classical exponential-weights algorithms and betting-based sequential decision frameworks~\cite{cover1999elements, cesa2006prediction,freund1997decision}. $\ell(\cdot)$ is the Gaussian log-score $\ell(y;\mu,\sigma^2)
=
\frac{1}{2}
\left[
    \log(2\pi\sigma^2)
    +
    \frac{(y-\mu)^2}{\sigma^2}
\right]$,
i.e., the negative log-likelihood of $y$ under the normal model $\mathcal{N}(\mu,\sigma^2)$. Note the Gaussian log-score is standard in universal portfolios~\cite{cover1991universal}; alternative scores (Beta, Brier~\cite{gneiting2007strictly}) are left to future work.
The normalized trust weights are then obtained via softmax normalization,
\begin{equation}\label{eq:softmax_pi}
    \pi_t^{(k)}
    =
    \frac{\exp(s_t^{(k)})}{\sum_{j=1}^K \exp(s_t^{(j)})}.
\end{equation}

Importantly, this update does not seek to identify a simulator that matches the real world; rather, it maintains a distribution over simulators that are most informative for predicting real-world outcomes, even when all simulators are misspecified.

Putting the above components together yields Algorithm~\ref{alg:exp_betting}, an explicitly executable instantiation of Algorithm~\ref{alg:abstract_betting}. At each round, simulator experts are updated via a universal score-weighted mixture, their predictive moments are aggregated into $(m_t,v_t)$, a Kelly-style bet $b_t$ is computed using~\eqref{eq:approx-kelly}, and the real-world payoff is incorporated into the bet-weighted estimator $\tau_t$. No other changes are made to Algorithm~\ref{alg:abstract_betting} concerning the abstract betting loop, payoff, or estimator; only the approximation of the Kelly quantities is specified.


Finally, although Algorithm~\ref{alg:exp_betting} assumes a fixed simulator bank, an important direction for future work is to adaptively revise simulators when performance degrades, for instance when sustained declines in wealth are observed. As in finance, such behavior may indicate model misspecification, but in sim-to-real settings the appropriate update mechanism is inherently application-dependent. This leads to the final challenge addressed in this section: how do we know, in practice, when the approximate bet is actually working?

\begin{algorithm}
    \begin{algorithmic}[1]
\State {\bf Given:} $P$, $\{Q_\theta\}_{k=1}^K$ for some $K \in \mathbb{Z}$, $\psi: \mathcal{X} \to \mathcal{M}$, $T \in \mathbb{Z}, \eta>0$
\State {\bf Initialize:} $\mathcal{D} = \emptyset$, $\tau_0 \in \mathcal{M}$, $W=1$, $S_0=\emptyset$, $s_0^{(k)}=0$, $\pi_0^{(k)}, \forall k$ s.t. $\sum_{k=1}^K \pi_0^{(k)}=1$
\State {\bf For} $t = 1, 2, \ldots, T$:
\State \ \ \ \ Sample $\tilde{S}_t \sim Q_\theta$, $S_t \leftarrow S_{t-1} \cup \tilde{S}_t$ 
\State \ \ \ \ Based on $S_t$ and $\mathcal{D}$:
\State \ \ \ \ \ \ \emph{compute $m_t$, $v_t$, $b_t$ by \eqref{eq:approx-kelly}}
\State \ \ \ \ \ \ \emph{compute bet direction $\text{sign}(m_t - \tau_{t-1})$}
\State \ \ \ \ {\bf If} $b_t > 0$:
\State \ \ \ \ \ \ \ \ line 10-14 in Algorithm~\ref{alg:abstract_betting}
\State \ \ \ \ \ \ \ \ \emph{Update log-score $s_t^{(k)}$ by \eqref{eq:score_update} for all $k\in\mathbb{Z}_K$}
\State \ \ \ \ \ \ \ \ \emph{Update weights $\pi_t^{(k)}$ by \eqref{eq:softmax_pi} for all $k\!\in\!\mathbb{Z}_K$}
\State {\bf Output:} $\hat{\mu}_{BW} = \tau_t$
\end{algorithmic}
    \caption{Approximated Betting}
    \label{alg:exp_betting}
\end{algorithm}

\subsection{When does betting help, approximately?}\label{subsec:method-assess}
One possible approach is to invoke Theorem~\ref{thm:efficiency}, provided that the required variance and bias terms can be estimated with provable guarantees. While this may be feasible in principle, we were unable to obtain such estimates in this paper and will consider it as future work. The answer we provide is remarkably simple yet deeply characteristic of the betting perspective: \emph{wealth}. 

\begin{theorem}[Growing wealth implies effective estimate, statistically]
\label{thm:wealth-evalue}
Let $(\mathcal{F}_t)_{t=0}^T$ be a filtration and let $(Y_t)_{t=1}^T$ be the realized payoff of the betting strategy at round $t$ (as defined in Algorithm~\ref{alg:abstract_betting}, line 12), with $Y_t \geq -1$ almost surely. At each round $t$, a predictable betting strategy (measurable w.r.t.\ $\mathcal{F}_{t-1}$) chooses a stake $b_t \in [0,1]$, then observes $Y_t$ and updates the wealth via $W_0 = 1, W_{t+1} = W_t \bigl(1 + b_t Y_t\bigr), t=0,\dots,T-1$. Suppose the following \emph{no-advantage (no-edge) null hypothesis} holds:
\begin{equation}\label{eq:no-edge-null}
    H_0:\quad
    \mathbb{E}[Y_t \mid \mathcal{F}_{t-1}] \le 0
    \quad\text{for all }t=1,\dots,T,
\end{equation}
i.e., conditioned on the past, the betting strategy has no systematic positive expected gain.
Then for every $\alpha\in(0,1)$,
\begin{equation}\label{eq:e-advantage}
    \mathbb{P}_{H_0}\bigl(W_T \ge 1/\alpha\bigr) \;\le\; \alpha.
\end{equation}
\end{theorem}


The theorem itself relies only on the no-edge condition~\eqref{eq:no-edge-null} and makes no assumptions about how the bets are generated~\cite{vovk1995game,ramdas2023game}. When combined with Theorem~\ref{thm:efficiency} and the Kelly inverse-variance equivalence, this motivates interpreting $H_0$ as the ``bad'' regime in which the betting strategy has no predictive advantage and the bet-weighted estimator offers no improvement over the Monte Carlo estimator. 

\begin{remark}[Interpreting wealth growth]\label{rmk:intepret-wealth}
The guarantee~\eqref{eq:e-advantage} should not be interpreted as a probability of outperforming Monte Carlo (i.e., observing wealth $w$ does not imply improvement with probability $1-1/w$). Instead, sustained wealth growth provides increasing evidence that the betting strategy is extracting predictive signal from the data, indicating operation in a regime where bet-weighted estimation is expected to outperform plain Monte Carlo. A more refined probabilistic interpretation, for example via sequential factorization of wealth growth, is left to future work.
\end{remark}

\subsection{Discussions}
The proposed method addresses a different aspect of the performance-evaluation problem than variance reduction (e.g., IS) or bias-correction (e.g., PPI) approaches. IS relies on likelihood ratios reweighting the sampling process itself, while bias-correction methods, such as PPI, use real data to calibrate or debias estimates so that simulator outputs approximate the real distribution. Our betting framework instead treats simulators as sources of predictive advantage: simulator bias is neither estimated nor removed, but explicitly tolerated and traded against variance when informative edge is present. From this perspective, betting complements IS and bias-correction methods by operating in regimes where simulator fidelity is limited but simulation is abundant. In particular, unlike existing approaches that are bottlenecked by scarce real-world data, the proposed framework directly embraces arbitrarily large, even mismatched, simulator banks and provably benefits from their diversity. It could be further \emph{combined} with IS or PPI, though details are of future interest. The method involves only two interpretable hyperparameters: the learning rate $\eta \in \mathbb{R}_{\ge 0}$ and the Kelly fraction $\lambda \in (0,1]$. Empirical sensitivity to their choices is examined in Section~\ref{sec:case}. A more complete theoretical characterization is of future interest. 

\section{Case Studies}\label{sec:case}
The empirical studies illustrate the theoretical insights of Section~\ref{sec:method}, highlighting both effective and failure regimes. Accordingly, our comparisons focus primarily on the Monte Carlo and various proposed betting-based estimators, including the ideal Kelly strategy described in Section~\ref{subsec:method-kelly} and its practical approximations implemented in Algorithm~\ref{alg:exp_betting}. Additional comparisons with IS and PPI variants are provided in the supplementary material. As both require case-specific adaptations and are not applicable to all scenarios, our most careful fair implementations show consistent advantages for the proposed method. Programs in Python that allow direct replications of the presented results can be found in the GitHub repository~\footnote{\href{https://github.com/ISUSAIL/Bet4Sim2Real}{https://github.com/ISUSAIL/Bet4Sim2Real}}. 


\begin{figure*}
    \begin{subfigure}{.55\textwidth}
    \includegraphics[width=0.99\textwidth]{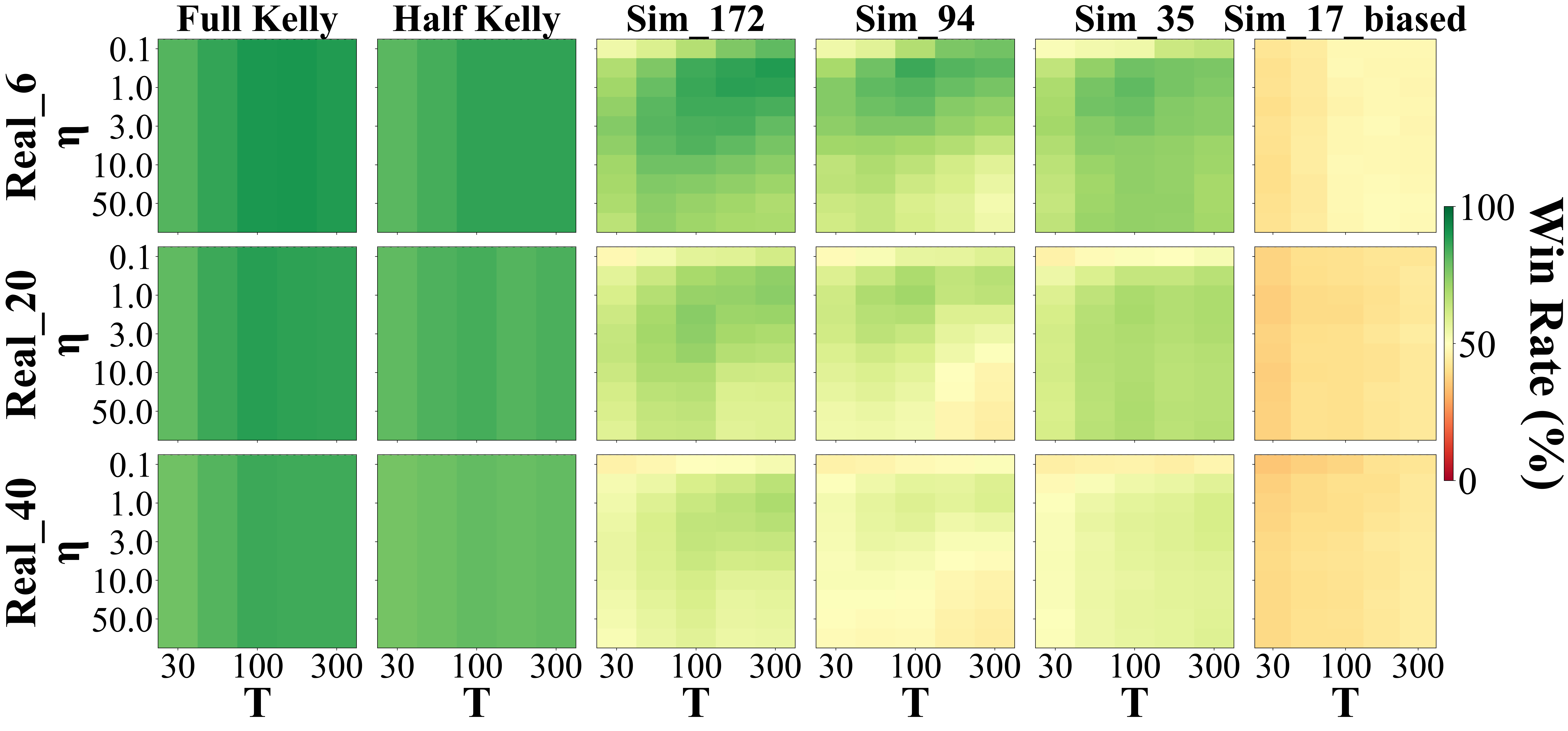}
    \caption{}
    \label{fig:synthetic:winrate}
    \end{subfigure}%
    \begin{subfigure}{.4\textwidth}
\includegraphics[width=.99\textwidth]{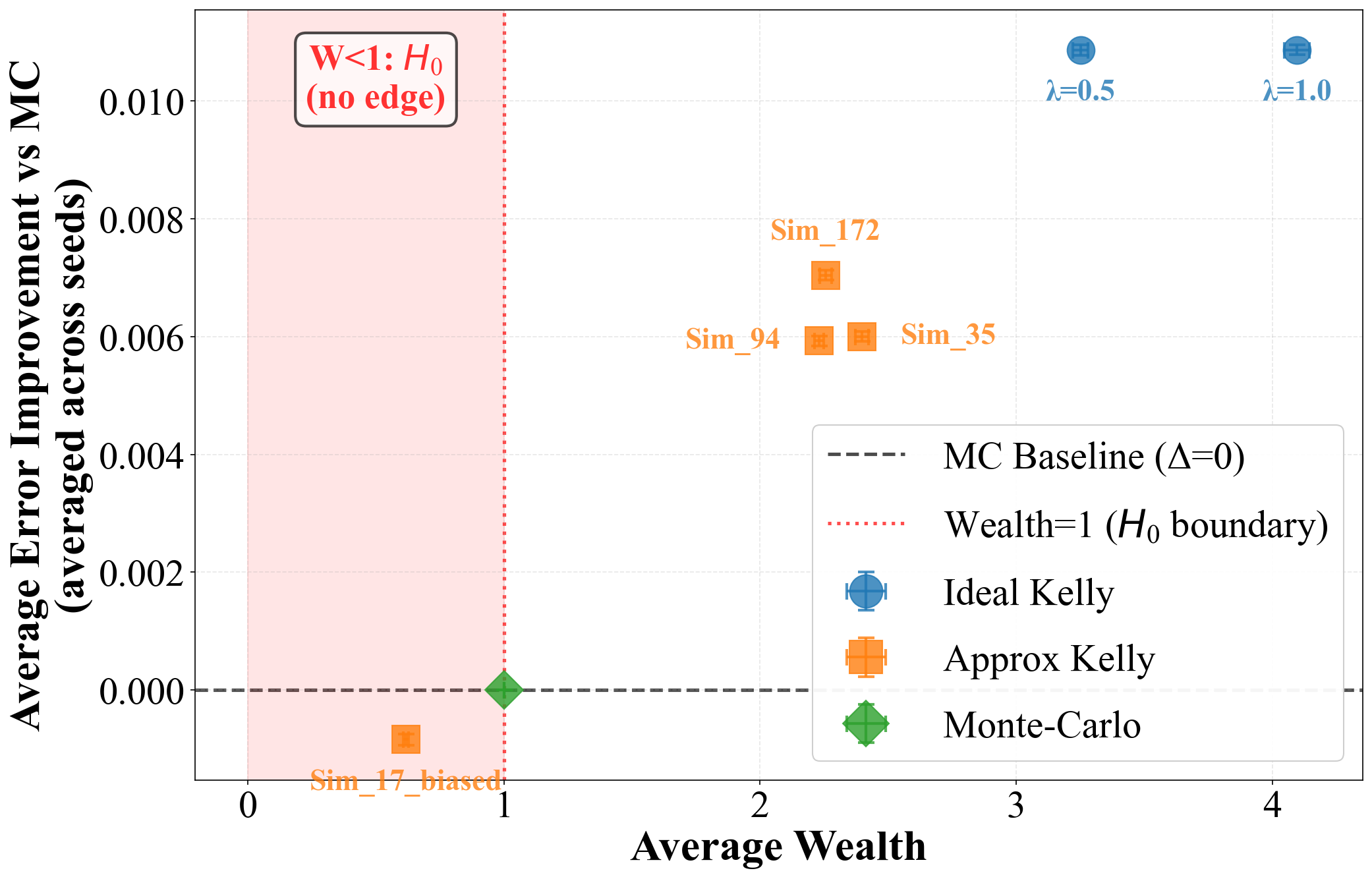}
    \caption{}
    \label{fig:synthetic:wealth}
    \end{subfigure}

    \begin{subfigure}{.45\textwidth}
    \includegraphics[width=.99\textwidth]{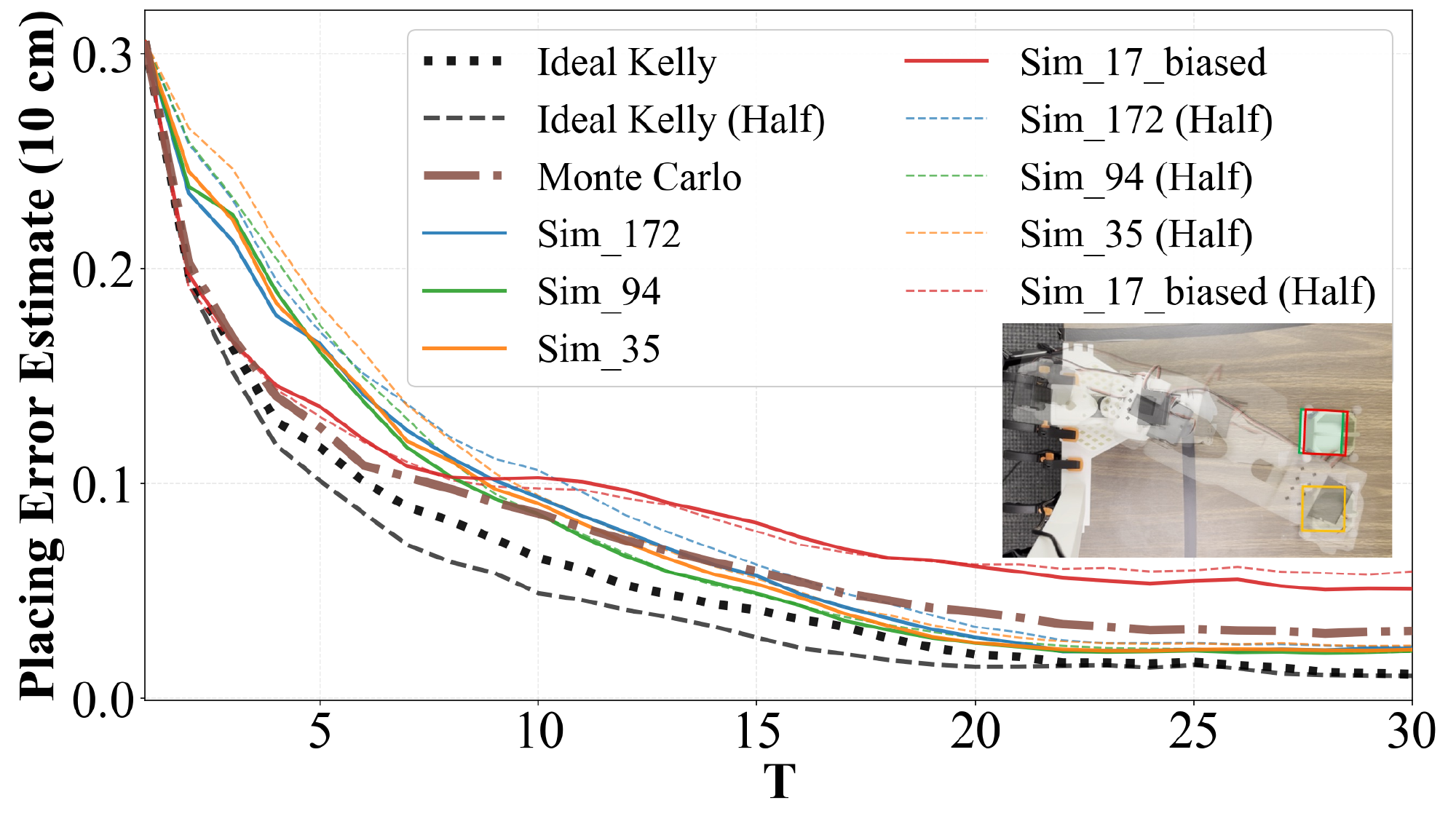}
    \caption{}
    \label{fig:pnp}
    \end{subfigure}%
    \hfill
    \begin{subfigure}{.45\textwidth}
    \includegraphics[width=.99\textwidth]{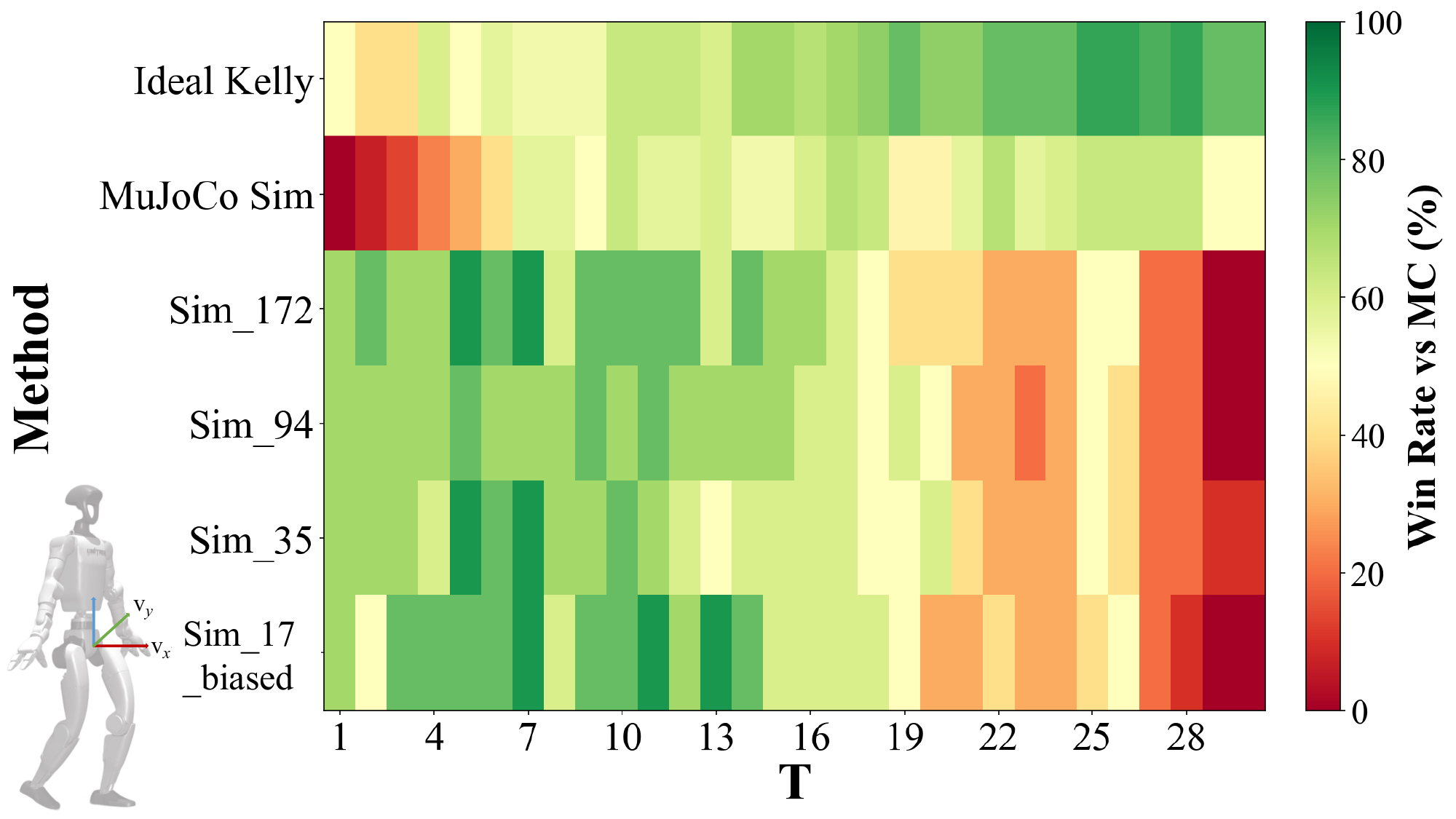}
    \caption{}
    \label{fig:loco}
    \end{subfigure}
    \vspace{-2mm}
    \caption{\footnotesize{Experiment results for Section~\ref{sec:case}: \ref{fig:synthetic:winrate}: win-rate comparison of all methods against the Monte Carlo baseline across different \texttt{Real} synthetic distributions, learning rates, and numbers of rounds; \ref{fig:synthetic:wealth}: average wealth across methods on the \texttt{Real\_6} distributions, providing empirical support for Theorem~\ref{thm:wealth-evalue}; \ref{fig:pnp}: Round-by-round estimates of the placement error for Section~\ref{sec:real_pnp} across different methods, the \emph{(Half)} label denotes half Kelly betting ($\lambda=0.5$); \ref{fig:loco}: Round-by-round win rate comparison for Section~\ref{sec:sim_loco} across different methods. The maximum number of rounds is set to $T=30$, reflecting commonly adopted trial counts in standardized real-world robotic testing, particularly for manipulators~\cite{roth1976performance,iso9283,iso18646}.}}
    \label{fig:exp}
    \vspace{-6mm}
\end{figure*}

\subsection{Synthetic examples}\label{sec:synthetic}

This section characterizes the real-world distribution $P$ and simulators $Q_{\theta}$ using a diverse family of parameterized synthetic distributions (e.g., Beta distributions, truncated normal distributions, Gaussian mixtures, Bernoulli, bimodal, and uniform-spike distributions (details can be found in the affiliated code and supplementary material) adapted from~\cite{Weng2025Rethink}) over the $[0,1]$ interval. The performance measure adopted in these cases is $\psi(x)=x$. 


\paragraph{Configurations}
For the ``real-world'' distributions, we consider three collections of synthetic distribution variants with cardinalities 6, 20, and 40, denoted as \texttt{Real\_6}, \texttt{Real\_20}, and \texttt{Real\_40}, respectively. Each collection is a strict subset of the next, constructed incrementally by progressively enriching the support. For the banks of simulator experts used in the approximated Kelly betting, we construct four collections of synthetic distribution variants with sizes 17, 35, 94, and 172, denoted as \texttt{Sim\_17\_biased}, \texttt{Sim\_35}, \texttt{Sim\_94}, and \texttt{Sim\_172}. \texttt{Sim\_17\_biased} is intentionally skewed and biased toward the left portion of the distributional support, whereas the other three banks provide increasingly dense coverage over the entire range with varying resolutions. Importantly, none of the distribution variants in the \texttt{Sim} banks overlap with those used in the \texttt{Real} collections. For the implementation of ideal Kelly variants, we adopt the ground-truth mean and ground truth variance (which upper bounds the conditional variance $\sigma_t$).

For each evaluated method (Monte Carlo, ideal Kelly with different Kelly fraction values, and approximated Kelly under different configurations), and for each task (corresponding to a specific distribution within one of the \texttt{Real} collections), the experiment is repeated over 100 independent random seeds. For each seed, the estimation error is evaluated against the ground-truth mean, as dictated by the known parameters of the underlying distribution. A bet-weighted estimate is considered a ``win'' if its estimation error is smaller than that of the Monte Carlo estimator on the same task, using the same random seed and the same number of rounds~$T$ (i.e., the same number of real-world samples). The overall results of the win rate evaluation is shown in Fig.~\ref{fig:synthetic:winrate} related to a variety of learning rates $\eta \in \{0.1,0.5,1.0,2.0,3.0,4.0,5.0,10.0,20.0,50.0,100.0\} $ for \eqref{eq:score_update} and different rounds $T \in \{30,50,100,200,300\}$. Each cell in the heatmap of Fig.~\ref{fig:synthetic:winrate} reports the win rate of a given method on a specific task, computed over 100 random seed trials. For the approximated Kelly variants, the win rate is evaluated for the corresponding choice of $\eta$ and number of rounds~$T$. Finally, for the evaluation against \texttt{Real\_6}, the average wealth achieved by each Kelly betting variant, aggregated over all 100 random seed trials and across different choices of $\eta$ and $T$, is reported in Fig.~\ref{fig:synthetic:wealth}. The red shaded region indicates the no-edge null hypothesis dictated by~\eqref{eq:no-edge-null} in Theorem~\ref{thm:wealth-evalue}.

\paragraph{Analysis and Discussions}
Overall, the empirical results support the theoretical claims in Section~\ref{sec:method} from multiple perspectives. As shown in Fig.~\ref{fig:synthetic:winrate}, ideal Kelly betting consistently dominates Monte Carlo by a clear margin across all settings, with performance improving as the number of rounds $T$ increases due to more stable estimation from additional real-world samples. This trend aligns with the asymptotic nature of the theoretical analysis. For the approximated Kelly method (Algorithm~\ref{alg:exp_betting}), the largest simulator bank, \texttt{Sim\_172}, generally achieves the best performance; however, more experts do not necessarily imply better results, as \texttt{Sim\_35} often outperforms \texttt{Sim\_94}. The learning rate~$\eta$ also influences performance, though these effects are descriptive for the studied cases, and a systematic characterization of their roles remains an important direction for future work. The wealth analysis in Fig.~\ref{fig:synthetic:wealth} empirically supports Theorem~\ref{thm:wealth-evalue}.

\subsection{From synthetic simulators to real-world pick-and-place}\label{sec:real_pnp}
This section evaluates the pick-and-place positioning accuracy of an imitation-learning-based policy on a SO-ARM101 manipulator~\cite{soarm100_github} (the policy was trained from tele-operated trajectories collected on the same physical robot, and therefore no naturally ``matching'' simulator exists for this task~\cite{Zhao-RSS-23}). 

\paragraph{Configuration}
In each trial, a green bucket is placed at a fixed location and moved to a target position offset by $102$~mm, with a maximum trial duration of 60~s. Positioning accuracy is measured as the Euclidean distance between the target center and the bucket center using the OptiTrack motion capture system with five Prime$^{\text{x}}$ 22 high-speed cameras ($\leq 0.15$~mm accuracy). To enforce an i.i.d.\ setting, each trial includes a full power cycle and a cooldown time $\ge 60~s$. The ground-truth error is estimated via Monte Carlo using 119 trials, achieving a relative half-width of $0.13$ at a $95\%$ confidence level. To align the performance measure~$\psi$ between the tracking task and the synthetic simulators, the placement error is expressed in units of $10$~cm and normalized (to $[0,1]$ interval) by a maximum distance of $30$~cm. For the Monte Carlo baseline, we evaluate 90 trials, each constructed from a contiguous window of 30 samples spanning the $i$-th to the $(i+30)$-th samples of the aforementioned ground-truth process. For each trial, we apply ideal Kelly betting with full and half Kelly fractions, as well as the approximated Kelly method using the same synthetic \texttt{Sim} variants as in the previous example. The estimation error, averaged over all trials as a function of $T$, is shown in Fig.~\ref{fig:pnp}.

\paragraph{Analysis and Discussion} 
Overall, ideal Kelly remains dominant, with half Kelly being more stable in the early stages, although both variants empirically converge to similar estimation error levels. Most \texttt{Sim} variants outperform Monte Carlo, particularly in later rounds as simulator weights adapt over time. In contrast, \texttt{Sim\_17\_biased} remains ineffective due to its pronounced mismatch with the task distribution.

\begin{remark}
Although synthetic distributions are effective in both examples, the pick-and-place study in particular highlights a key property of the bet-weighted estimator: its effectiveness depends on \emph{distributional} signal rather than task-specific or application-level alignment. This observation should not be interpreted as advocating the replacement of application-specific simulators with synthetic ones. In practice, synthetic banks often require broad distributional coverage and may be less efficient than well-correlated sim-to-real simulators, a point further illustrated in Section~\ref{sec:sim_loco}.
\end{remark}

\subsection{Sim-to-sim locomotion tracking performance evaluation}\label{sec:sim_loco}
In this section, we demonstrate the proposed method by evaluating the velocity-tracking performance of a Unitree G1 humanoid robot controlled by a reinforcement-learning-based locomotion policy~\cite{unitreegithub,schwarke2025rsl} in the MuJoCo simulator~\cite{todorov2012mujoco}. 

\paragraph{Configuration}: Each test trial lasts 4~s, starting from a fixed initial pose and executing two sequential velocity commands of 2~s each, specified by desired planar velocities $(v_x, v_y) \in [-1,1]^2$. The velocity-tracking error is defined as the mean $\ell_2$-norm of the difference between measured and commanded velocities over the 4~s horizon, normalized by a maximum speed of $1.5$~m/s. In this example, the ``real-world'' is a fixed MuJoCo environment with specified parameters (robot mass, damping coefficients, friction, and a fixed random seed for velocity commands). The simulator side consists of a domain-randomized bank of $10$ experts with varying parameters and random seeds, none matching the real instance. The experts' estimate update follows Algorithm~\ref{alg:exp_betting}. We also reuse the synthetic \texttt{Sim} variants from Section~\ref{sec:synthetic}. The ground-truth tracking error is estimated via sufficient sampling using 552 trials, achieving a relative half-width below $0.02$ at a $95\%$ confidence level. We then evaluate bootstrapped trials using different betting-based methods. Each method is repeated over 30 independent random seeds. As in Section~\ref{sec:synthetic}, a method is counted as a ``win'' if its error estimate outperforms Monte Carlo on the same samples. Results are shown in Fig.~\ref{fig:loco}, with learning rate $\eta=5.0$ throughout.

\paragraph{Analysis and Discussion}
Ideal Kelly betting remains the dominant method overall. Among the two classes of approximated Kelly variants, synthetic \texttt{Sim} distributions often perform better in the early rounds, largely due to noisy or ``lucky'' early guesses. In contrast, the bank of domain-randomized MuJoCo-based simulators may start with less favorable initial estimates, but converge rapidly and dominate in later rounds. Comparison on this specific case against a PPI variant~\cite{badithela2025reliable} can be found in the supplementary material.

\section{Conclusion: Limitations, and Future Work}
\label{sec:conclusion}

Beyond the theoretical limitations discussed in Remark~\ref{rmk:intepret-wealth}, the practical betting scheme in Section~\ref{subsec:method-abstract} currently lacks theoretical guarantees characterizing how well and to what extent it approximates the ideal Kelly strategy. While such guarantees appear attainable in principle (e.g., by leveraging classical regret bounds, mixture optimality results, or asymptotic efficiency analyses from Cover’s universal portfolio framework~\cite{cover2002universal}), formal development is left to future work. Another limitation (also common in robot testing) is the i.i.d. sampling assumption, which may not naturally hold in robotic experiments and can require additional effort to enforce (e.g., Section~\ref{sec:real_pnp}). Extending the framework to non-stationary settings with independent but non-identically distributed samples is an important direction for future work. On the empirical side, our demonstrations do not yet include large-scale, computationally intensive simulators (e.g., those with rich visual sensing or high-dimensional perception). While our theory applies broadly to scalar mean performance measures (and thus remains compatible with dimension-reduced measurements), extending the empirical evaluation to these settings is an important direction for future work.

\bibliographystyle{unsrtnat}
\bibliography{references}

\input{supplement}

\end{document}

%% file: supplement.tex
\clearpage
\onecolumn
\noindent{\LARGE Supplementary Material:``Betting for Sim-to-Real Performance Evaluation''}

\ 

This document supplements the paper titled ``Betting for Sim-to-Real Performance Evaluation'' by providing complementary proofs, extended experimental results, and additional discussions. Section, equation, figure, and algorithm numbering continues from the main paper without restarting. References are shared with and continue from the main paper. All experiments and figures can be reproduced through the released code base~\footnote{\href{https://github.com/ISUSAIL/Bet4Sim2Real}{https://github.com/ISUSAIL/Bet4Sim2Real}}. 

\section{Proof of Theorem~\ref{thm:efficiency}}\label{sup:var-proof}
\noindent\begin{proof}
The mean squared error of the bet-weighted estimator decomposes as
\begin{equation}
\mathbb{E}[(\hat{\mu}_{\text{BW}} - \mu)^2] = \text{Var}(\hat{\mu}_{\text{BW}}) + \text{Bias}^2(\hat{\mu}_{\text{BW}}).
\end{equation}

For the variance term, since the samples are conditionally independent given the weights
\begin{equation}
\text{Var}(\hat{\mu}_{\text{BW}}) = \text{Var}\left(\sum_{t=1}^n w_t \psi(X_t)\right) = \sum_{t=1}^n w_t^2 \text{Var}(\psi(X_t) | w_t) = \sum_{t=1}^n w_t^2 \sigma_t^2.
\end{equation}

For the bias term
\begin{equation}
\begin{aligned}
\text{Bias}(\hat{\mu}_{\text{BW}}) &= \mathbb{E}[\hat{\mu}_{\text{BW}}] - \mu
= \mathbb{E}\left[\sum_{t=1}^n w_t \psi(X_t)\right] - \mu
= \sum_{t=1}^n w_t \mathbb{E}[\psi(X_t) | w_t] - \mu
= \sum_{t=1}^n w_t (\mu + \beta(w_t)) - \mu\\
&= \sum_{t=1}^n w_t \beta(w_t).
\end{aligned}
\end{equation}

Therefore
\begin{equation}
\mathbb{E}[(\hat{\mu}_{\text{BW}} - \mu)^2] = \sum_{t=1}^n w_t^2 \sigma_t^2 + \left[\sum_{t=1}^n w_t \beta(w_t)\right]^2.
\end{equation}

For Monte Carlo with i.i.d. samples, we have $\mathbb{E}[(\hat{\mu}_{\text{MC}} - \mu)^2] = \sigma^2/n$. The bet-weighted estimator achieves lower MSE when:
\begin{equation}
\mathbb{E}[(\hat{\mu}_{\text{BW}} - \mu)^2] < \mathbb{E}[(\hat{\mu}_{\text{MC}} - \mu)^2]
\Leftrightarrow \sum_{t=1}^n w_t^2 \sigma_t^2 + \left[\sum_{t=1}^n w_t \beta(w_t)\right]^2 < \frac{\sigma^2}{n}
\Leftrightarrow \frac{\sigma^2}{n} - \sum_{t=1}^n w_t^2 \sigma_t^2 > \left[\sum_{t=1}^n w_t \beta(w_t)\right]^2.
\end{equation}

This gives the stated condition.
\end{proof}

\section{Proof of Theorem~\ref{thm:kelly-equivalence}}\label{sup:kelly-proof}
\noindent\begin{proof}
The variance-reduction term of Theorem~\ref{thm:efficiency},
\[
\frac{\sigma^2}{T}-\sum_{t=1}^T w_t^2\sigma_t^2,
\]
is maximized by minimizing $\sum_{t=1}^T w_t^2\sigma_t^2$ subject to $\sum_{t=1}^T w_t=1$ and $w_t\ge 0$. The unique minimizer is the inverse-variance weights $w_t\propto 1/\sigma_t^2$.

By the Taylor-expanded Kelly rule~\eqref{eq:kelly-small-stakes-inv-var}, there exists a constant $C>0$ such that
\begin{equation}
b_t = C\,\frac{\mathbb{E}[Y_t\mid\mathcal{F}_{t-1}]}{\mathrm{Var}(Y_t\mid\mathcal{F}_{t-1})}
= C\,\frac{\kappa_{t-1}}{\sigma_t^2},
\end{equation}
where $\kappa_{t-1}:=\mathbb{E}[Y_t\mid\mathcal{F}_{t-1}]$ is $\mathcal{F}_{t-1}$-measurable. Therefore the induced normalized weights are
\begin{equation}
w_t=\frac{b_t}{\sum_{j=1}^T b_j}
=
\frac{\frac{\kappa_{t-1}}{\sigma_t^2}}{\sum_{j=1}^T \frac{\kappa_{j-1}}{\sigma_j^2}}
\;\propto\;
\frac{\kappa_{t-1}}{\sigma_t^2},
\end{equation}
which proves the first claim.

In particular, if the predictive edge is (approximately) constant across rounds, $\kappa_{t-1}\approx \kappa$, then the Kelly-induced weights reduce (approximately) to inverse-variance weighting, $w_t\propto 1/\sigma_t^2$, proving the final claim.
\end{proof}

\section{Proof of Theorem~\ref{thm:wealth-evalue}}\label{sup:proof-wealth}
\noindent\begin{proof}
Conditioning on the past, we have
\begin{equation}
    \mathbb{E}[W_{t+1}\mid \mathcal{F}_{t-1}] 
    = \mathbb{E}\bigl[ W_t (1 + b_t Y_t) \mid \mathcal{F}_{t-1} \bigr] 
    = W_t \bigl(1 + b_t\,\mathbb{E}[Y_t \mid \mathcal{F}_{t-1}] \bigr) 
    \;\le\; W_t,
\end{equation}
    
where we used $b_t \in [0,1]$ and \eqref{eq:no-edge-null}.
Thus $(W_t)_{t=0}^T$ is a nonnegative supermartingale with $W_0=1$, and Ville's inequality further implies
\[
    \mathbb{P}_{H_0}\!\Bigl(\sup_{t \le T} W_t \ge 1/\alpha\Bigr) \leq \alpha.
\]
In particular,
\[
    \mathbb{P}_{H_0}(W_T \ge 1/\alpha) \leq \alpha,
\]
which completes the proof.
\end{proof}

\section{Extended Case Study Results}\label{sup:case}
Extended results for Section~\ref{sec:synthetic} are shown in Fig.~\ref{fig:synthetic:err} (with the $\eta=5.0$ sub-batch of the same win rate results shown in Fig.~\ref{fig:synthetic:winrate}). Readers are encouraged to check out the affiliated code base for more experiment results and hyper-parameter configurations. The \texttt{Sim} and \texttt{Real} synthetic distributions adopted in Section~\ref{sec:case} are illustrated in Fig.~\ref{fig:syn-sim-dist} and Fig.~\ref{fig:syn-real-dist}, respectively. Fig.~\ref{fig:pnp-extend} shows some extended and complementary results to the pick-and-place tests presented in Section~\ref{sec:real_pnp}. 

\begin{figure*}[h]
    \centering
    \includegraphics[width=0.99\textwidth]{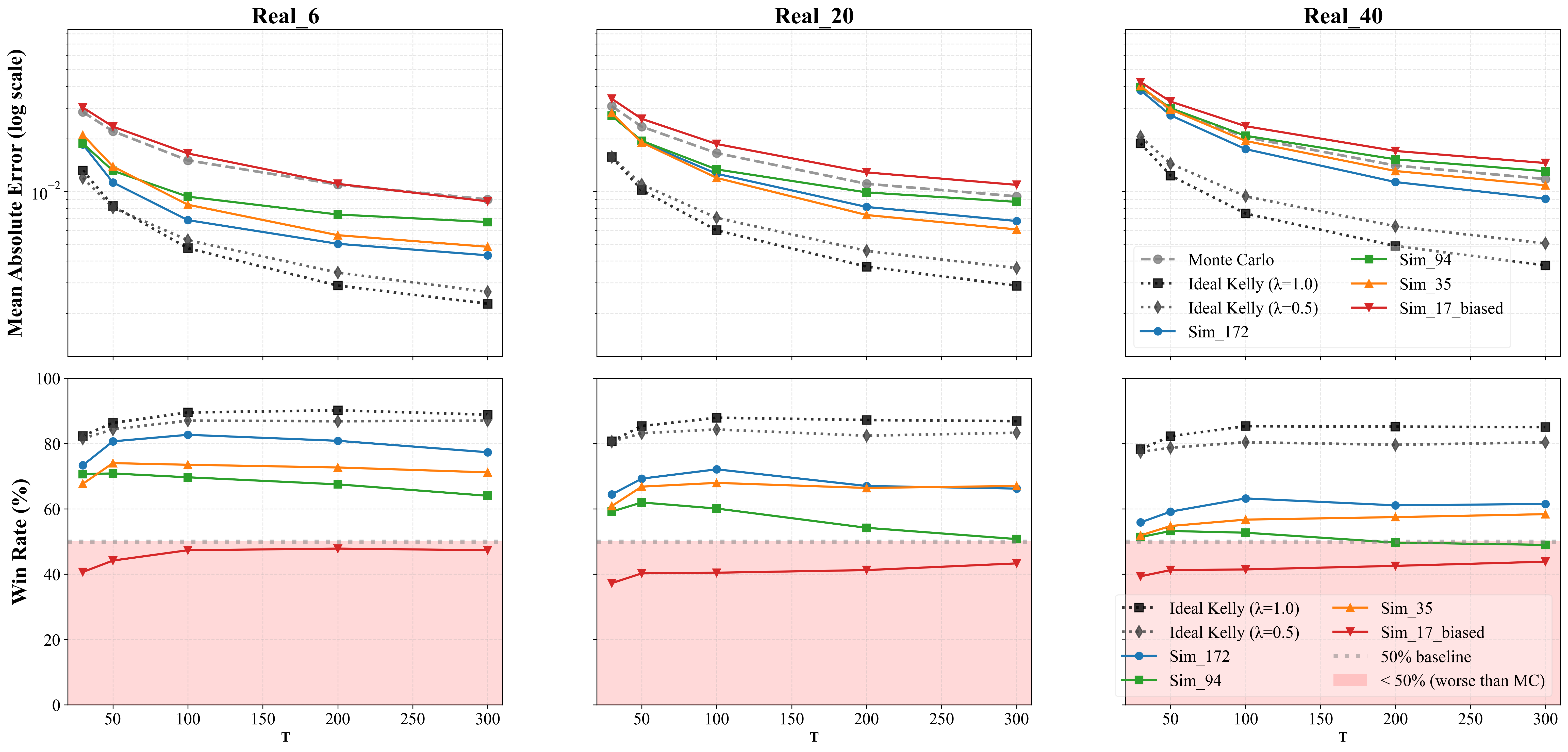}
    \caption{Extended performance comparison results among various methods discussed in Section~\ref{sec:synthetic} with a fixed learning rate parameter of $\eta=5.0$.}\label{fig:synthetic:err}
\end{figure*}

\begin{figure*}
    \centering
    \includegraphics[width=0.99\textwidth]{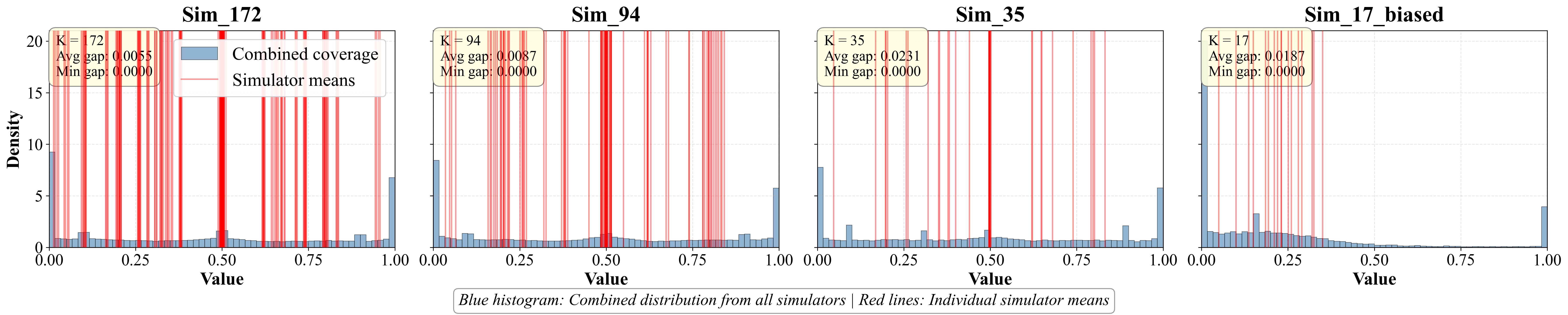}
    \caption{A combined illustration of different variants of banks of \texttt{Sim} distributions.}
    \label{fig:syn-sim-dist}
\end{figure*}
\begin{figure*}
    \centering
    \begin{subfigure}{0.99\textwidth}
        \includegraphics[width=0.99\textwidth]{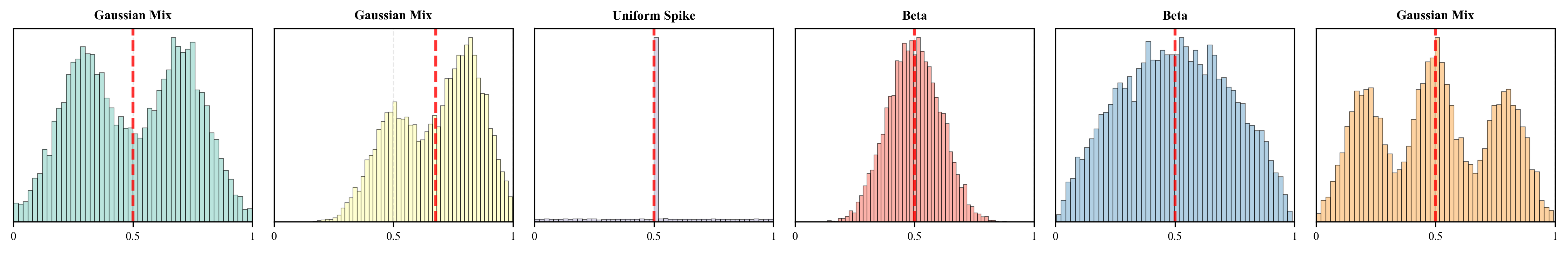}
        \caption{Real\_6}
        \label{fig:real-6}
    \end{subfigure}
    \begin{subfigure}{0.99\textwidth}
        \includegraphics[width=0.99\textwidth]{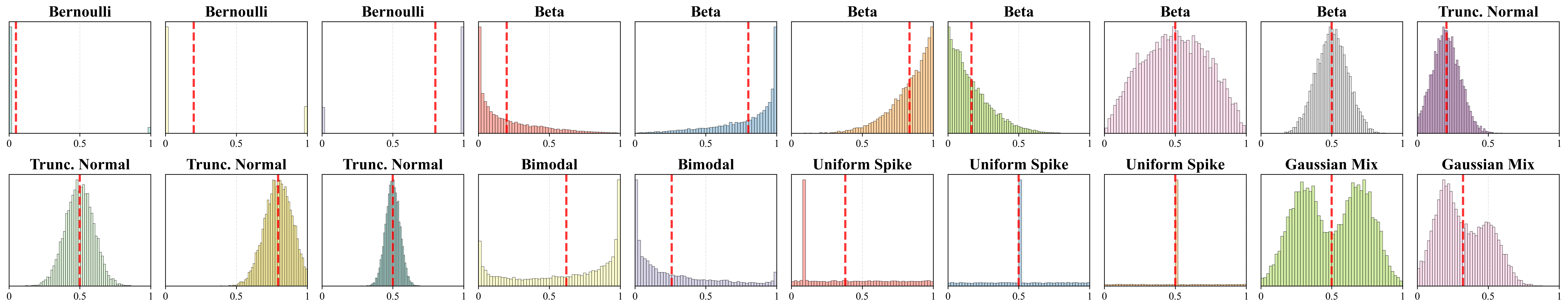}
        \caption{Real\_20}
        \label{fig:real-20}
    \end{subfigure}
    \begin{subfigure}{0.99\textwidth}
        \includegraphics[width=0.99\textwidth]{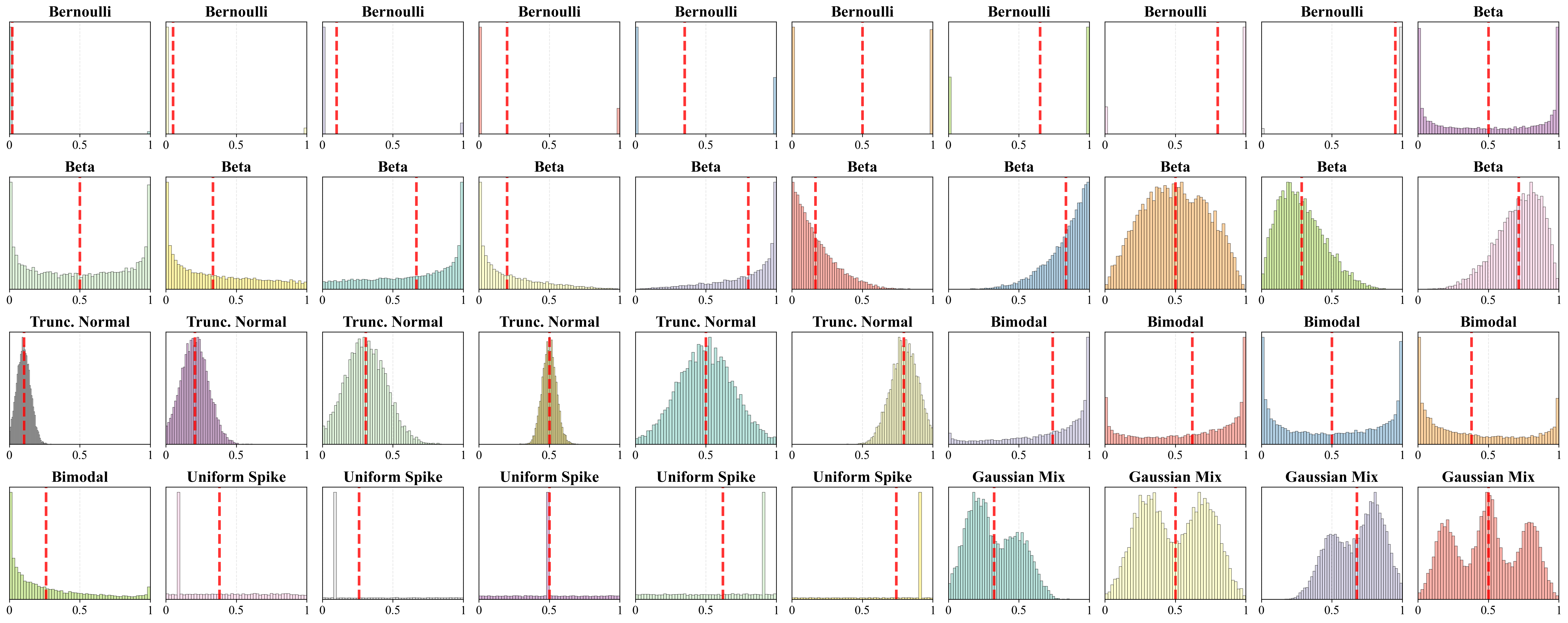}
        \caption{Real\_40}
        \label{fig:real-40}
    \end{subfigure}
    \caption{Synthetic \texttt{Real} distributions used in Section~\ref{sec:synthetic}.}
    \label{fig:syn-real-dist}
\end{figure*}

Comparing Fig.~\ref{fig:syn-sim-dist} and Fig.~\ref{fig:syn-real-dist} provides additional empirical insight into why \texttt{Sim\_172} dominates in Section~\ref{sec:synthetic}: its mass is concentrated near the most commonly observed \texttt{Real} distributions. The comparison also helps explain why \texttt{Sim\_35} outperforms \texttt{Sim\_94}, as the former exhibits slightly more uniform coverage of the relevant region, despite having a lower overall density.

\begin{figure}
    \centering
    \begin{subfigure}{.5\textwidth}
        \includegraphics[width=0.99\textwidth]{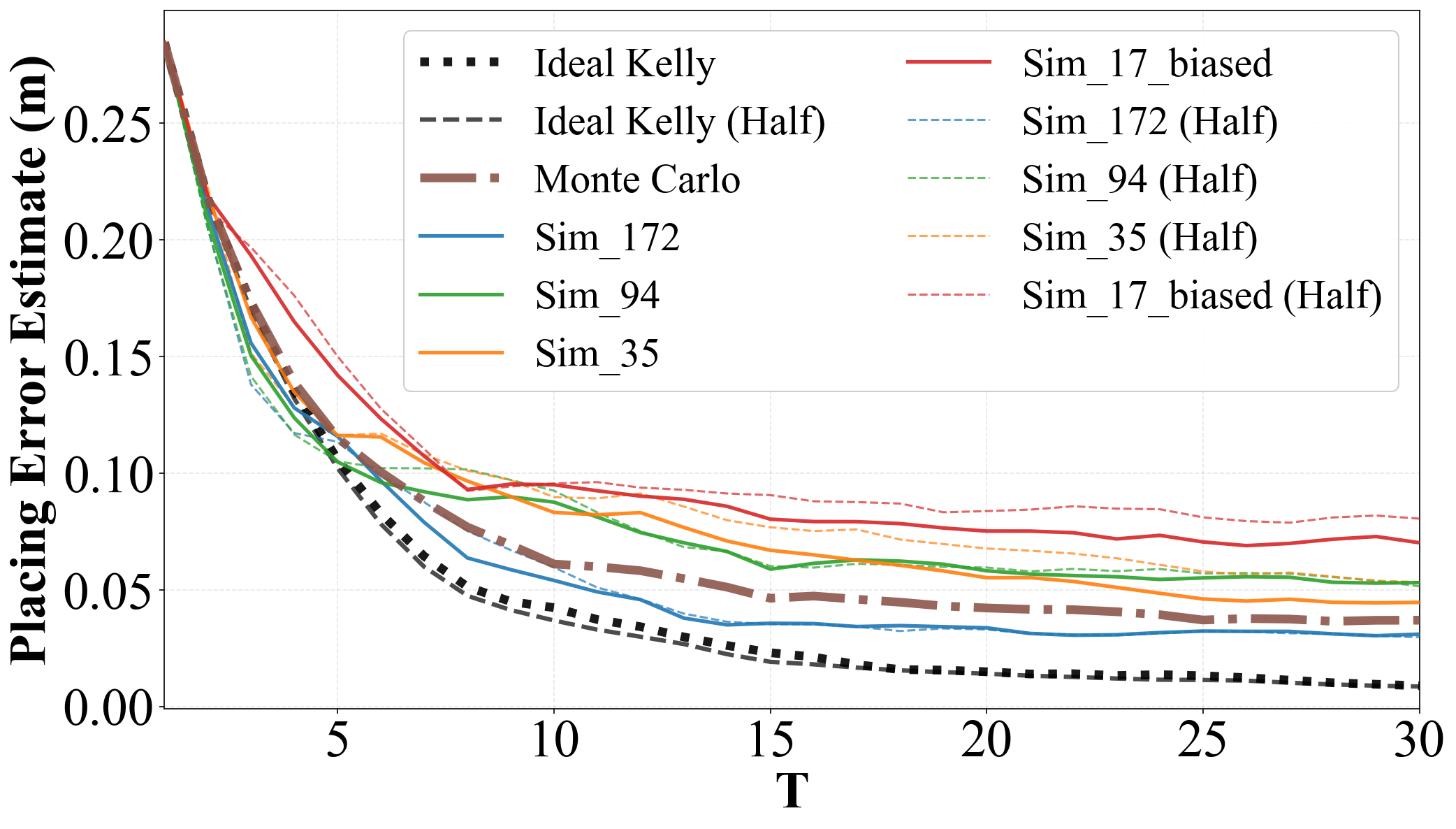}
        \caption{Extended experiments on another policy evaluated following the same procedure as described in Section~\ref{sec:real_pnp}. With the performance being significantly different from the one shown in the paper, \texttt{Sim\_172} becomes the most well-performed approximated Kelly variants.}\label{fig:policy2_pnp}
    \end{subfigure}%
    \hfill
    \begin{subfigure}{.3\textwidth}
        \includegraphics[width=0.99\textwidth]{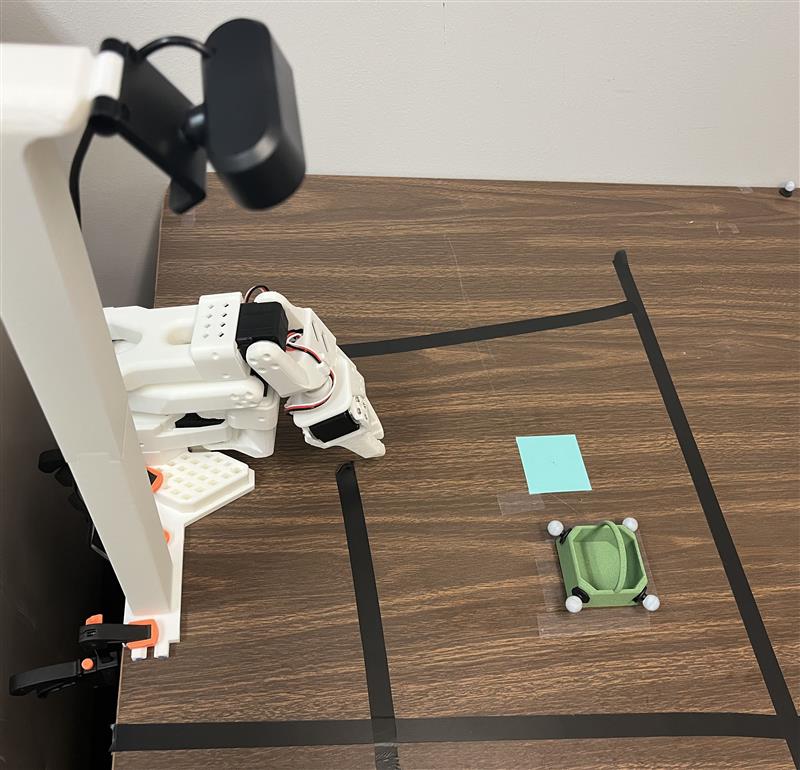}
        \caption{The real-world experiment setup of the pick-and-place manipulation accuracy testing discussed in Section~\ref{sec:real_pnp}.}\label{fig:pnp_setup}
    \end{subfigure}
    \caption{Extended and complementary results to Section~\ref{sec:real_pnp}.}\label{fig:pnp-extend}
\end{figure}

\subsection{Comparisons with a PPI variant~\cite{badithela2025reliable}}

\begin{figure}
    \centering
    \includegraphics[width=0.5\linewidth]{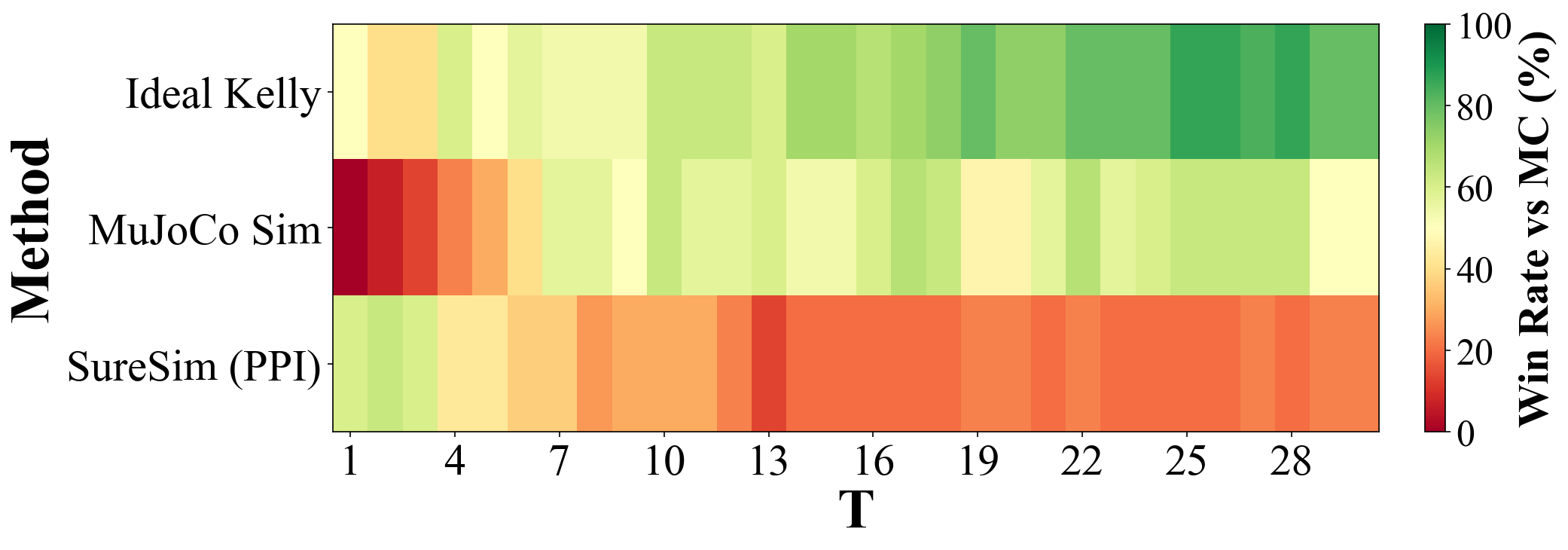}
    \caption{A win-rate comparison with SureSim~\cite{badithela2025reliable} on the locomotion task in Section~\ref{sec:sim_loco} is shown under the same fair experimental settings used in our reproduction.}
    \label{fig:suresim}
\end{figure}

A comparison with SureSim~\cite{badithela2025reliable}, a representative PPI-based method, is illustrated in Fig.~\ref{fig:suresim}. Since the original SureSim work does not provide an official release of its source code, we implemented our own best-effort reproduction following the descriptions in the paper. The case considered is the one in Section~\ref{sec:sim_loco}. Note that the synthetic examples in Section~\ref{sec:synthetic} are not compatible with SureSim, since no pairwise sim-real match exists by construction. The cases in Section~\ref{sec:sim_loco} are also incompatible, as there is no matched simulator available in the first place.

In Section~\ref{sec:sim_loco}, the MuJoCo-based approximated Kelly approach uses a bank of $10$ expert simulators, where each expert constructs its individual estimate from $5$ simulator samples, resulting in a total of $50$ simulator samples per estimate. To ensure a fair comparison between the MuJoCo-based approximated Kelly method and the SureSim framework, we allocate the same total simulator budget to SureSim. Specifically, SureSim uses up to $30$ simulator samples for pairwise testing to estimate the bias (matching the $30$ ``real'' samples), and an additional $20$ simulator samples for its internal estimation procedure.

Under this configuration, as shown in Fig.~\ref{fig:suresim}, SureSim underperforms both the standard Monte Carlo estimator and the proposed approximated betting method in terms of estimation accuracy. In a slightly unfair comparison by allowing SureSim to have significantly more sim-real pairwise tests or more simulator samples (e.g., changing the \texttt{N\_ADDITIONAL} constant in the released code in \texttt{locomotive\_tracking/demo\_suresim.py} from 20 to 3200), one would have a win rate better than Monte Carlo and on par with approximated Kelly betting methods.

We further emphasize, however, the following important considerations; for these reasons, this additional experiment is provided for completeness but was not included in the main paper, as the page limit would not allow a sufficiently careful and fair discussion of these nuances:
\begin{enumerate}
    \item The sim-real pairwise testing required by SureSim may be limited by the relatively small number of samples available in our setting (30 samples), and the additional sim-only samples (20 samples) may also be insufficient to fully realize its potential (as mentioned above). Nevertheless, this allocation reflects the most balanced and fair use of the available simulator budget for our comparison.
    \item SureSim involves a larger number of hyperparameters that may require careful tuning; in our reproduction, we did not perform extensive hyperparameter optimization.
    \item A primary strength of SureSim (and PPI-based methods more broadly) lies not in producing the most accurate point estimate of the mean, but in providing confidence intervals with guaranteed coverage. From this perspective, direct point-estimate comparison may not fully reflect its intended use, though it remains the most practical basis for comparison in our setting.
    \item SureSim is primarily designed around a single simulator and relies on correlation-based adjustments, whereas the proposed Kelly-style betting variants naturally accommodate and benefit from a diverse bank of simulators.
    \item The two approaches are not mutually exclusive. As discussed in the paper, PPI-style bias correction and betting-based variance reduction address complementary aspects of the sim-to-real inference problem and could potentially be combined in future work.
\end{enumerate}

\subsection{Comparisons with IS}
The practical implementation of IS (importance sampling) is highly case-specific. In this study we contextualize the performance of our Kelly betting approach comparing against the ``theoretically'' optimal IS for mean estimation across the synthetic distributions studied in Section~\ref{sec:synthetic}. 

For the mean estimation of~\eqref{eq:mu}, we sample from a proposal importance distribution $q(x)$ and compute the estimator by~\eqref{eq:is}. The variance-minimizing proposal distribution for estimating $\mathbb{E}_p[f(X)]$ is commonly known as~\cite{asmussen2007stochastic,o2018scalable}:
\begin{equation}
    q^*(x) = \frac{|\psi(x) - \mathbb{E}_p[\psi(X)]| \cdot p(x)}{\mathbb{E}_p[\psi(X) - \mathbb{E}_p[\psi(X)]|]}.
\end{equation}

For mean estimation of the presented synthetic distributions where $\psi(x) = x$, this becomes:
\begin{equation}
    q^*(x) = \frac{|x - \mu| \cdot p(x)}{Z}, \quad Z = \mathbb{E}_p[|X - \mu|]
\end{equation}
This is the \emph{zero-variance} importance distribution in the asymptotic sense.

\begin{figure}
    \centering
    \includegraphics[width=0.5\linewidth]{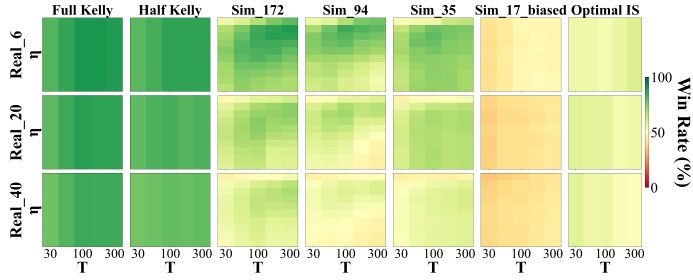}
    \caption{Extending the case studies of Fig.~\ref{fig:synthetic:winrate} to comparisons against the ``optimal'' IS settings.}
    \label{fig:is}
\end{figure}

For each distribution in \texttt{Real\_6}, \texttt{Real\_20}, and \texttt{Real\_40}, we sample from its corresponding $q^*(x)$ using rejection sampling with a uniform proposal on $[0,1]$. Each trial collects $T=300$ samples and is repeated over 100 random seeds. Win rates against Monte Carlo follow the same protocol as in Section~\ref{sec:synthetic}. The resulting IS performance is reported as an extended column from Fig.~\ref{fig:synthetic:winrate} in Fig.~\ref{fig:is}. Note that this setup gives IS a strong oracle advantage (even stronger than ideal Kelly): it has access to the true mean $\mu$ and the target distribution $p$.

Despite this oracle information, optimal IS achieves only $\sim$50-55\% win rate against standard Monte Carlo. In contrast, our Kelly betting approach achieves 70-100\% win rates (ideal Kelly) and 60-80\% win rates (approximated Kelly with simulator banks), substantially outperforming optimal IS. We attribute this discrepancy to the following factors:
\begin{enumerate}
    \item The self-normalized IS estimator~\eqref{eq:is} is biased at finite sample sizes, even when $q=q^*$. The zero-variance guarantee is asymptotic: while variance vanishes as $n\rightarrow\infty$, both bias and variance can remain non-negligible for practical budgets (here $n\leq 300$).
    \item Unlike IS, which draws samples from a fixed proposal, Kelly betting is sequential and adaptive. This adaptivity allows it to incorporate early outcomes and progressively allocate weight toward uncertainty reduction.
\end{enumerate}
As discussed in the main paper, the proposed Kelly-style betting mechanism is not intended to replace IS or debiasing methods such as PPI. Rather, by viewing sim-to-real performance evaluation from a complementary perspective that explicitly values simulator-derived predictive edge, it has strong potential to integrate with IS/PPI in hybrid variance-reduction pipelines; we leave such integrations to future work.

\section{Terminology from betting and finance}\label{sup:terminology}
Throughout this paper we adopt a few standard terms from the literature on sequential betting and finance. For clarity, we summarize them here in plain language.
\begin{itemize}
    \item \textbf{Bet}: A ``bet'' refers simply to choosing how strongly to rely on a particular prediction at a given round.  In our setting, placing a bet corresponds to allocating a fraction $b_t\in[0,1]$ of our testing budget toward the prediction that the next real-world outcome will exceed (or fall below) the current threshold. There is no monetary interpretation required. This is merely a device for quantifying confidence.

    \item \textbf{Payoff}: The ``payoff'' $Y_t$ is an abstract signal indicating whether the prediction at over
    round $t$ was helpful or misleading. A positive payoff means the prediction pointed in the correct direction (by a regressional margin as suggested in Algorithm~\ref{alg:abstract_betting}); a negative payoff means it did not. In the performance evaluation terms in this paper, $Y_t$ measures whether the simulator-informed guess about the next outcome was directionally correct.

    \item \textbf{Edge}: The ``edge'' of a betting strategy refers to having a systematic predictive advantage, i.e., $\mathbb{E}[Y_t\mid\mathcal{F}_{t-1}]>0$. An edge means that, on average, the information being used (e.g., simulator predictions) helps us anticipate real-world outcomes better than chance. ``No edge'' simply means no useful predictive signal.

    \item \textbf{Wealth}: ``Wealth'' $W_t$ is a bookkeeping variable used to track the cumulative success or failure of the bets. Importantly, \emph{wealth is not a weight on simulators}; it is a diagnostic quantity that summarizes how informative the betting strategy has been so far.  The key theoretical fact is that wealth forms an $e$-process under a no-edge null hypothesis (Theorem~\ref{thm:wealth-evalue}).
\end{itemize}

\newpage
\section{}\label{apx:nomenclature}

\nomenclature[B]{$P$}{Real-world probability distribution}
\nomenclature[B]{$Q_{\theta}$}{Simulator distribution parameterized by $\theta$}
\nomenclature[B]{$\theta$}{Simulator parameter}
\nomenclature[B]{$\Theta$}{Parameter space of $\theta$}
\nomenclature[B]{$\mathcal{X}$}{Outcome space}
\nomenclature[B]{$\psi(\cdot)$}{Performance-measure outcome function}
\nomenclature[B]{$\mu$}{Target mean, $\mu \triangleq \mathbb{E}_{x\sim P}[\psi(x)]$}

\nomenclature[B]{$\hat{\mu}_{\mathrm{MC}}$}{Monte Carlo estimator of $\mu$}
\nomenclature[B]{$n$}{Number of real-world i.i.d. samples used in various estimators}
\nomenclature[B]{$\hat{\mu}_{\mathrm{IS}}$}{Importance-sampling estimator}
\nomenclature[B]{$P'$}{Auxiliary (importance) distribution used for importance sampling}
\nomenclature[B]{$p(\cdot)$}{Density / mass function of $P$ (w.r.t. a shared base measure)}
\nomenclature[B]{$p'(\cdot)$}{Density / mass function of $P'$}
\nomenclature[B]{$\hat{\mu}_{\mathrm{BC}}$}{Bias-correction based estimator of $\mu$}
\nomenclature[B]{$b_{\theta}(\cdot)$}{Bias-correction term added to implicit/explicit simulator outputs}

\nomenclature[C]{$x_t$}{Real-world sample at round $t$}
\nomenclature[C]{$y_t$}{Observed scalar outcome, typically $y_t=\psi(x_t)$}
\nomenclature[C]{$T$}{Number of real-world rounds/trials/samples (horizon)}
\nomenclature[C]{$\mathcal{F}_{t-1}$}{Information available up to round $t-1$}
\nomenclature[C]{$\sigma_t^2$}{Conditional variance, $\mathrm{Var}(\psi(x_t)\mid\mathcal{F}_{t-1})$}
\nomenclature[C]{$\kappa_{t-1}$}{Predictive edge factor (an $\mathcal{F}_{t-1}$-measurable signal term)}

\nomenclature[D]{$\tau_{t-1}$}{Running reference/baseline used in payoff definition}
\nomenclature[D]{$B_t$}{Bet direction indicator (e.g., sign of predicted advantage)}
\nomenclature[D]{$b_t$}{Bet size (stake) at round $t$}
\nomenclature[D]{$Y_t$}{Payoff at round $t$}
\nomenclature[D]{$\lambda$}{Kelly fraction (risk-scaling parameter), $\lambda\in(0,1]$}
\nomenclature[D]{$w_t$}{Normalized bet weight (e.g., $w_t=b_t/\sum_{j=1}^T b_j$)}
\nomenclature[D]{$\hat{\mu}_{\mathrm{BW}}$}{Bet-weighted estimator of $\mu$}
\nomenclature[D]{$C$}{Constant for Taylor-expanded small-stake Kelly rule}
\nomenclature[D]{$W$}{Wealth}

\nomenclature[D]{$K$}{Number of simulators in the simulator bank}
\nomenclature[D]{$k$}{Simulator index, $k\in\{1,\dots,K\}$}
\nomenclature[D]{$\mu^{(k)}$}{Mean prediction associated with simulator $k$}
\nomenclature[D]{$(\sigma^{(k)})^2$}{Variance prediction associated with simulator $k$}
\nomenclature[D]{$s_t^{(k)}$}{Cumulative score of simulator $k$ up to round $t$}
\nomenclature[D]{$\pi_t^{(k)}$}{Trust weight of simulator $k$ at round $t$}
\nomenclature[D]{$\eta$}{Learning rate (inverse-temperature) in score/trust updates}
\nomenclature[D]{$\ell(\cdot)$}{Gaussian log-score used to evaluate predictions}
\nomenclature[D]{$m_t$}{Mixture predictive mean (trust-weighted)}
\nomenclature[D]{$v_t$}{Mixture predictive variance proxy (trust-weighted)}

\nomenclature[F]{$W_t$}{Wealth process induced by sequential betting}
\nomenclature[F]{$\alpha$}{Significance level used in the wealth-based statistical statement}

\printnomenclature